\documentclass{article}

\usepackage[colorlinks,
linkcolor=purple,
anchorcolor=blue,
citecolor=blue
]{hyperref}

\PassOptionsToPackage{numbers, compress}{natbib}

\usepackage[preprint]{neurips_2024}

\usepackage[utf8]{inputenc} % allow utf-8 input
\usepackage[T1]{fontenc}    % use 8-bit T1 fonts
\usepackage{hyperref}       % hyperlinks
\usepackage{url}            % simple URL typesetting
\usepackage{booktabs}       % professional-quality tables
\usepackage{amsfonts}       % blackboard math symbols
\usepackage{nicefrac}       % compact symbols for 1/2, etc.
\usepackage{microtype}      % microtypography
\usepackage{xcolor}         % colors
\usepackage{colortbl}

\usepackage{amssymb}
\usepackage{amsmath}
\usepackage{algorithm}
\usepackage{algpseudocode}

\usepackage{multirow}
\usepackage{bbm}
\usepackage{amsthm}
\usepackage{enumitem}
\usepackage{mathtools}  % Extended package of amsmath
\usepackage{amsmath} 

\usepackage{pstricks,pst-node}
\usepackage{comment} 
\usepackage{microtype}
\usepackage{graphicx}
\usepackage{caption}
\usepackage{wrapfig}
\usepackage{xspace}
\usepackage{tikz}

\usepackage{subcaption}
\usepackage{booktabs} % for professional tables

\usepackage{listings}
\usepackage{xcolor}
\definecolor{codegreen}{rgb}{0,0.6,0}
\definecolor{codegray}{rgb}{0.5,0.5,0.5}
\definecolor{codepurple}{rgb}{0.58,0,0.82}
\definecolor{backcolour}{rgb}{0.95,0.95,0.92}

\lstdefinestyle{overleafstyle}{
    backgroundcolor=\color{backcolour},   
    commentstyle=\color{codegreen},
    keywordstyle=\color{magenta},
    numberstyle=\tiny\color{codegray},
    stringstyle=\color{codepurple},
    basicstyle=\ttfamily\footnotesize,
    breakatwhitespace=false,         
    breaklines=true,                 
    captionpos=b,                    
    keepspaces=true,                 
    numbers=left,                    
    numbersep=5pt,                  
    showspaces=false,                
    showstringspaces=false,
    showtabs=false,                  
    tabsize=2
}

\newtheorem{theorem}{Theorem}[section]

\newtheorem{proposition}{Proposition}

\newcommand{\cf}{x^\text{cf}}

\newcommand{\codeurl}{\url{https://github.com/BirkhoffG/CFMark}}
\newcommand*\circled[1]{\textcircled{\raisebox{-0.9pt}{#1}}}

\newcommand\norm[1]{\left\lVert#1\right\rVert}

\newcommand{\TITLE}{Watermarking Counterfactual Explanations}

\title{\TITLE}

\author{%
Hangzhi Guo,~Firdaus Ahmed Choudhury,~Tinghua Chen,~Amulya Yadav \\
  Penn State University\\
  \texttt{\{hangz,fac5186,tuc579,amulya\}@psu.edu} \\
}

\begin{document}

\maketitle

\begin{abstract}
  % The field of Explainable Artificial Intelligence (XAI) focuses on techniques for providing explanations to end-users about the decision-making processes that underlie modern-day machine learning (ML) models. 
% Within the vast universe of XAI techniques, counterfactual (CF) explanations are often preferred by end-users as they help explain the predictions of ML models by providing an easy-to-understand \& actionable recourse (or contrastive) case to individual end-users who are adversely impacted by predicted outcomes. 
% Counterfactual (CF) explanations, a popular Explainable AI (XAI) technique, provide actionable recourse recommendations to those who are adversely impacted by predicted outcomes.
%----Proposed change------
% Counterfactual (CF) explanations, a popular Explainable AI (XAI) technique, provide actionable recourse recommendations to those who are adversely impacted by predicted outcomes. CF Explanations attempt to find the smallest modifications to the feature values of an instance that would change the prediction of the ML model on that instance to a predefined output.
%  Counterfactual (CF) explanations, a popular Explainable AI (XAI) technique, provide actionable recourse recommendations by attempting to find the smallest modifications to the feature values of an instance that would change the ML models prediction on that instance to a predefined output.
Counterfactual (CF) explanations for ML model predictions provide actionable recourse recommendations to individuals adversely impacted by predicted outcomes.
However, despite being preferred by end-users, CF explanations have been shown to pose significant security risks in real-world applications;
in particular, malicious adversaries can exploit CF explanations to perform query-efficient model extraction attacks on the underlying proprietary ML model.
To address this security challenge, we propose {\emph{CFMark}}, a novel model-agnostic watermarking framework for detecting unauthorized model extraction attacks relying on CF explanations.
{\emph{CFMark}} involves a novel bi-level optimization problem to embed an indistinguishable watermark into the generated CF explanation such that any future model extraction attacks using these watermarked CF explanations can be detected using a null hypothesis significance testing (NHST) scheme. At the same time, the embedded watermark does not compromise the quality of the CF explanations. 
We evaluate {\emph{CFMark}} across diverse real-world datasets, CF explanation methods, and model extraction techniques. 
Our empirical results demonstrate  {\emph{CFMark}}'s effectiveness, achieving an F-1 score of $\sim$0.89 in identifying unauthorized model extraction attacks using watermarked CF explanations. Importantly, this watermarking incurs only a negligible degradation in the quality of generated CF explanations  (i.e., $\sim$1.3\% degradation in validity and 1.6\% in proximity).
Our work establishes a critical foundation for the secure deployment of CF explanations in real-world applications.

\end{abstract}

\section{Introduction}
% counterfactual explanations are great for users
Within the field of Explainable AI techniques, counterfactual (CF) explanations\footnote{
Counterfactual explanation \citep{wachter2017counterfactual} and algorithmic recourse \citep{ustun2019actionable} are closely connected \citep{verma2020counterfactual, stepin2021survey}. Hence, we use both terms interchangeably throughout this paper.} 
\citep{wachter2017counterfactual,mothilal2020explaining,karimi2021algorithmic,guo2023counternet} become a popular technique for explaining the predictions generated by machine learning (ML) models \citep{bhatt2020explainable,shang2022not}. Given an input instance $x$, CF explanation methods identify a similar counterfactual instance $\cf$ that would yield a different, often more desirable, prediction from the ML model. CF explanations are useful for offering recourse to vulnerable groups. For example, when an ML model spots a student as vulnerable to dropping out of school, CF explanation techniques can suggest corrective measures to teachers, who can intervene accordingly. 
%In addition, CF explanations can provide helpful advice to impoverished loan applicants who are rejected by an ML-based algorithm used by the lender \citep{wachter2017counterfactual}.
%For example, CF explanations can provide helpful advice to loan applicants who are rejected by an ML-based algorithm or actionable recommendations for teachers working with students who are at risk of dropping out of school. Therefore, CF explanations have emerged as an effective model explanation method that can meet the regulatory requirements imposed by governments worldwide, such as the EU's General Data Protection Regulation (GDPR) .

\begin{figure*}[ht!]
     \centering
     \begin{subfigure}[h]{0.83\textwidth}
         \centering
         \includegraphics[width=0.95\textwidth]{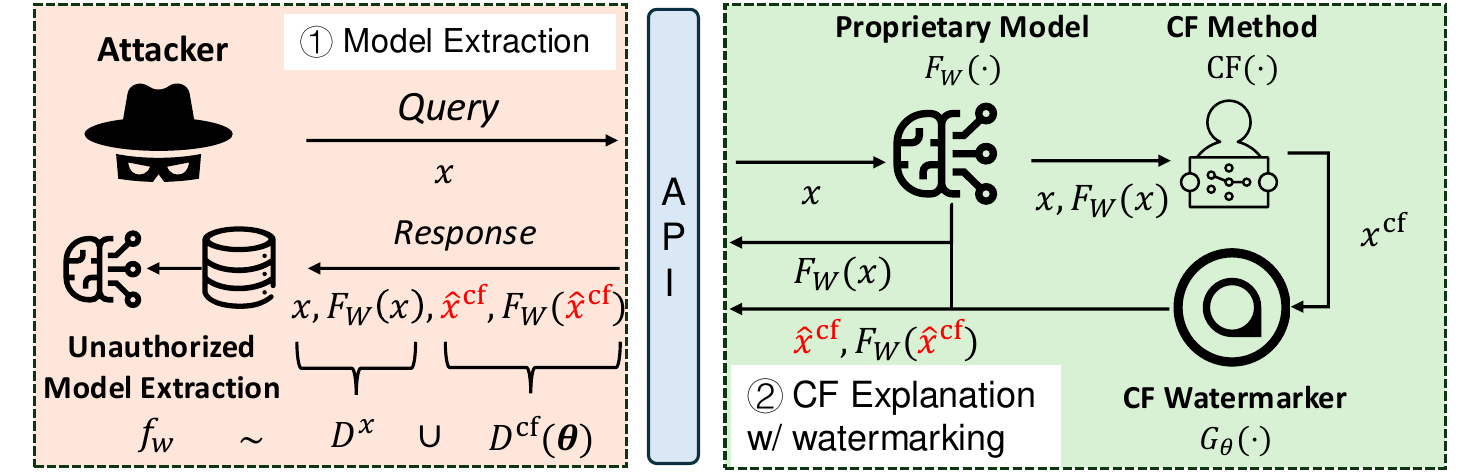}
         \caption{Illustration of \circled{1} model extraction attack using counterfactual explanations, and \circled{2} the procedure of generating watermarks of counterfactual explanations.}
         \label{fig:pipeline}
     \end{subfigure}
     \hfill
     \begin{subfigure}[h]{0.16\textwidth}
         \centering
         \includegraphics[width=0.95\textwidth]{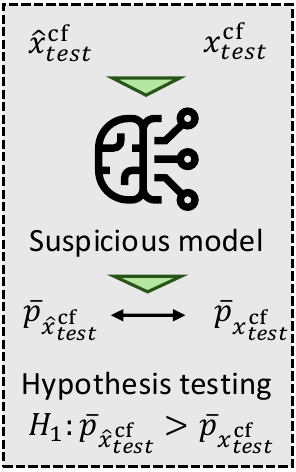}
         \caption{Model ownership verification.}
         \label{fig:ownership}
     \end{subfigure}
     \hfill
     \caption{\textbf{Illustration of CFMark.}
     (a) \circled{1} Model extraction Attack. The adversaries use querying data $D^x$ and their CF explanations $D^\text{cf}(\boldsymbol{\theta})$ to train a private model $f_w$ that reproduces the predictive behavior of the proprietary model $F_W$. 
     \circled{2} Watermarking. Given CF explanations $\cf$ and the proprietary model $F_W$, CFMark embeds a watermark into CF explanations $\hat{x}^\text{cf} = G_\theta(\cf)$.
     (b) Model ownership verification. If the adversaries use $D^\text{cf}(\boldsymbol{\theta})$ to train an extracted model $f_w$, our framework can \emph{identify} this unauthorized usage via hypothesis testing.}
    \label{fig:framework}
\end{figure*}

% However, despite the benefits, no real-world adoption in CF explanations. This is because the risk in the enterprise side for the adversarial vulnerability. Recent study shows models can be extracted through CF Explanations.
% These vulnerabilities prevent the adoption of CF explanations in real-world applications.
% However, despite these usability benefits, widespread real-world adoption of CF explanation techniques in practical ML systems remains limited.
% One key obstacle hindering this wider adoption is the risk of model extraction attacks through CF explanations \citep{sokol2019counterfactual}.

% Despite the usability benefits of CF explanations, concerns about the risk of model extraction attacks through CF explanations hinder their widespread adoption in practical ML systems \citep{sokol2019counterfactual}.
Despite these usability benefits, the widespread real-world adoption of CF explanation techniques in practical ML systems remains limited.
One key obstacle hindering this wider adoption is the risk of model extraction attacks through CF explanations \citep{sokol2019counterfactual}.
There is an inherent tension between explainability and security:
the transparency offered by CF explanations, while benefiting legitimate users, can be exploited by adversaries to extract the underlying proprietary model.
% while transparency through model explanations benefits legitimate users, this transparency can inadvertently be exploited by adversaries to extract underlying proprietary models.
As shown in Figure~\ref{fig:pipeline}, attackers can exploit CF explanations to execute model extraction attacks by using both the input instances and the corresponding CF explanations to train a surrogate model $f_w$ which reproduces the predictive behaviors of the proprietary model $F_W$.
This approach is more query-efficient than traditional model extraction attacks relying only on input-output pairs \citep{aivodji2020model,wang2022dualcf}, thus posing a serious model security challenge.

% CF explanations are data points that cross the ML model's decision boundary. 
% Consequentially, adversaries can extract the proprietary models by leveraging the input data points and their corresponding CF explanations. 
% Recent studies show that leveraging CF explanations achieves query-efficient model extraction attack to proprietary models with black-box access \citep{aivodji2020model,wang2022dualcf}.
% Given these security vulnerabilities, the service providers are reluctant to adopt CF explanations in real-world applications, even though CF explanations can bring benefits to users who are negatively impacted by the ML model's decisions.

% This is a difficult problem.
% Explain why this is difficult: perturb the CF explanations that are not learnable by the adversary, but useful for the users. This might be an inherently conflicting goal. (TODO: think more about this part)
% For this reason, very few prior work have addressed this problem. One of the few work use differential privacy, but it significantly reduces the utility of the explanations.
Limited research has been conducted on countermeasures against model extraction attacks that exploit CF explanations. While \citet{yang2022differentially} proposed using differential privacy in CF explanations to mitigate such attacks, their method relies on drastic perturbations to CF explanations to impede attackers, consequently reducing the explanations' utility.
Alternatively, digital watermarking presents a strong defense mechanism against model extraction attacks \citep{jia2021entangled,szyller2021dawn,kirchenbauer2023watermark} by embedding identifiable markers within the model or its training data, enabling subsequent model ownership verification.
However, existing digital watermarking methods do not explicitly address the vulnerabilities exposed through CF explanations, which makes existing defense mechanisms inadequate for mitigating the threat of model extractions using CF explanations.

% unique signals into the ML model and/or its training data, which enables subsequent verification of model ownership whenever an adversary tries to conduct a model extraction attack. However, existing digital watermarking techniques have not explicitly considered the use of CF explanations in model extraction attacks. Consequently, existing defense mechanisms are insufficient for mitigating the threat of model extraction attacks using CF explanations.

% Defending against the model extraction attack using counterfactual explanations can be challenging. This attack involves creating artificial inputs and querying the CF explanations to retrain an ML model that mimics the proprietary models owned by the enterprise. To prevent this, it is necessary to perturb these explanations aggressively so that adversaries cannot use them to train a good model. However, perturbations to the CF explanations can compromise their usability for legitimate users. The more perturbations are made, the less usable they become. For instance, \citet{yang2022differentially} suggest incorporating differential privacy into CF explanations to prevent model extraction attacks. However, this method significantly diminishes the utility of CF explanations.

% In this paper, instead of aggressively perturbing the CF explanations, we alternatively propose a watermarking techniques for counterfactual explanations to prevent model extraction.
% What is watermarking
% Why watermarking can prevent model extraction
% Alternatively, 

\textbf{Contributions.}
In this paper, we propose CFMark, a novel model-agnostic watermarking framework to safeguard against unauthorized model extraction attacks using CF explanations. 
CFMark involves two stages (illustrated in Figure~\ref{fig:framework}):
(i) \emph{watermark embedding}: the watermark function $G_\theta$ embeds a watermark into CF explanations $\hat{x}^\text{cf} = G_\theta(\cf)$ (see Figure~\ref{fig:pipeline} \circled{2}).
(ii) \emph{watermark detection}: we employ a pairwise t-test to identify suspected models trained on the watermarked CF explanations (see Figure~\ref{fig:ownership}).
Importantly, CFMark is model-agnostic and can be used with any CF method without compromising the utility of CF explanations.
% Importantly, compared with prior methods that aggressively perturb the CF explanations, this watermarking framework has little impact on the utility of CF explanations.
Our primary contributions are in three folds: 
% In this paper, we focus on \emph{identifying unauthorized model extractions}. We propose a novel watermarking framework of counterfactual explanations that embeds an indistinguishable watermark to CF explanations that can be later detected if malicious users use these watermarked CF explanations for training their models. Unlike prior defense methods that aggressively perturb the CF explanations, this watermarking framework has little impact regarding to the utility of CF explanations. Our primary contributions are summarized as follows: 
\begin{itemize}[leftmargin=*]
    \item (Problem-Wise) We propose a novel approach to combat the model extraction attacks using CF explanations. 
    Instead of sacrificing the quality of CF explanations for enhanced security,
    % Unlike prior work focusing on defending the model extraction attack by trading off the CF explanations' quality \citep{yang2022differentially}, 
    we propose to add watermarks to CF explanations that can provide easy identifiability of any ML model that is trained (by an adversary without authorization) using watermarked CF explanations as training data. To our knowledge, we are the first to consider watermarking CF explanations to prevent model extraction attacks.
    \item (Methodology-Wise) 
    We propose a model-agnostic framework for embedding watermarks in CF explanations for model ownership verification. This framework involves a bi-level optimization to generate watermarks that can be later identified via a pairwise t-test. We further theoretically analyze the effectiveness of this verification procedure.    
    % This framework is designed to operate in two stages: 
    % (1) First, we develop a technique that embeds watermarks in the CF explanations. We construct this watermark by solving a bi-level optimization problem, which can be identified at a later stage, while having little impact on the utility of CF explanations.
    % % Importantly, this watermark has little effect on the utility of counterfactual explanations. In addition, it is unidentifiable to malicious users, but it can be identified at a later stage. 
    % % We construct this watermark by solving a bi-level optimization problem.
    % (2) Second, we perform a pairwise t-test to identify models that are trained on the watermarked CF explanations. 
    \item (Experiment-Wise) We conduct an extensive evaluation of CFMark across various real-world datasets, CF explanation methods, and model extraction techniques. Our results show that our watermarking techniques can achieve reliable identifiability (0.89 F1-score), without trading off the utility of counterfactual explanations (only $\sim$1.3\% reduction in validity and $\sim$1.6\% in proximity).
\end{itemize}
% Contributions:
% 1. First to introduce watermarking techniques for counterfactual explanations.
% 2. Propose a model-agnostic techinque.
% 3. We show that our watermarking techniques can effectively prevent model extraction while maintaining the utility of the explanations.

\section{Related Work}
\label{sec:related}
\textbf{Counterfactual Explanation Techniques.}
% Extensive research on counterfactual (CF) explanation techniques primarily focuses on developing methods that yield different, often more desirable predicted outcomes \citep{wachter2017counterfactual, verma2020counterfactual, karimi2020survey}. 
Prior work on CF explanation techniques can be organized into two categories: (i) \emph{non-parametric methods} \citep{wachter2017counterfactual, ustun2019actionable, mothilal2020explaining, van2019interpretable, karimi2021algorithmic, upadhyay2021towards,verma2020counterfactual,karimi2020survey}, which typically find optimal CF explanations by solving an individual optimization or searching problem, and (ii) \emph{parametric methods} \citep{pawelczyk2020learning, yang2021model, mahajan2019preserving, guo2023counternet,guo2023rocoursenet,vo2023feature}, which adopt parametric models (e.g., a neural network model) to generate recourses. 
However, existing techniques fail to consider the security risks associated with providing CF explanations to end-users, which leaves the generated CF explanations vulnerable to adversaries to extract proprietary ML models.

\textbf{Security and Privacy Risks in CF Explanations.}
Recent research has highlighted the privacy and security risks associated with 
model explanations \citep{sokol2019counterfactual, shokri2021privacy}. In particular, CF explanations can be used to carry out model extraction attacks \citep{aivodji2020model, wang2022dualcf}, linkage attacks \citep{goethals2023privacy}, and membership inference attacks \citep{pawelczyk2023privacy}. To mitigate these risks, \citet{vo2023feature} use feature discretization to defend against linkage attacks, but this approach lacks generalizability to defend against other attacks, such as model extraction attacks. Alternatively, differentially private CF explanations exhibit resistance to model extraction and membership inference attacks \citep{yang2022differentially}. 
However, this approach directly deteriorates the quality of CF explanations for improved security, which limits its practicality in real-world applications.
% Unfortunately, this approach suffers from the severely decreased quality of CF explanations, which limits its practical deployment in real-world applications.

\textbf{Model Extraction Attacks and Watermarking.}
Our work is closely related to prior literature on model extraction (ME) attacks, which focuses on constructing private models that behave similarly to the proprietary victim model.  \citet{tramer2016stealing} first conceptualized this attack, and later work improves the efficacy of ME attacks via active learning \citep{chandrasekaran2020exploring, pal2020activethief}, semi-supervised learning \citep{jagielski2020high}, adversarial examples \citep{juuti2019prada,yu2020cloudleak}, and CF explanations \citep{aivodji2020model, wang2022dualcf}.
A common approach for protecting against model extraction attacks focuses on reducing the quality of ML models, such that it becomes unattractive for an adversary to conduct a model extraction attack \citep{tang2024modelguard}. Alternatively, digital watermarking techniques embed unique signals into either the training data \citep{li2022untargeted}, model parameters \citep{song2017machine, jia2021entangled}, and/or model outputs \citep{szyller2021dawn,kirchenbauer2023watermark}, which enables the defender (i.e., enterprise) to verify suspicious models constructed via a model extraction attack.
Unfortunately, existing watermarking techniques have not explicitly considered the use of CF explanations in model extraction attacks. %Hence, existing methods cannot protect against such attacks that exploit CF explanations.

% \section{THE PROPOSED FRAMEWORK: COUNTERNET}
\section{Preliminaries}
\label{sec:problem}

% \paragraph{Problme Setting.}
We focus on binary classification as it represents the most common setting in CF explanation research \cite{verma2020counterfactual, guo2023counternet}.
Let $D_t = \{(x_i,y_i)\}^{N}_{i=1}$ denote a training dataset containing $N$ data points, where $x_i \in \mathbb{R}^d$ represents the $i$-th input data point, and $y_i \in \{0, 1\}$ denotes its label. The enterprise service provider (i.e., defender) uses this training dataset $D_t$ to train their proprietary predictive model $F_W: \mathcal{X} \to [0,1]$, which outputs a probabilistic score $F_W(x)$ for a given input $x$.
% takes as input $x \in D_t$ to output a probabilistic score $F_W(x)$.% which represents the probability that the actual label for $x$ is 1.
%Importantly, we denote $F_W(x)$ as the probability output from this proprietary predictive model, as this probability prediction of the input $x$ is generated by $F$ parameterized by $W$.

\textbf{Counterfactual Explanations.}
In addition to producing predictions $F_W(x)$ of input $x$ using their predictive model, the service provider (i.e., defender) uses a CF explanation method (see Figure~\ref{fig:framework}) to generate counterfactuals $\cf$ that explain the predictions of $F_W$ on input $x$. 
% At a high level, finding $\cf$ (for input $x$) involves applying feature-space changes to $x$ that reverses the predictive outcome (from $F_W(x)$ to $1-F_W(x)$) with a minimal cost of change. 
Given an input $x$ and the proprietary model $F_W$, a CF explanation method $\verb|CF|(x;F_W)$ generates CF explanations $\cf$ which satisfies two criteria: 
(i) they need to be \emph{valid} \citep{wachter2017counterfactual,guo2023counternet}, i.e., the CF explanations get opposite predictions from the original input $F_W(\cf)=1-F_W(x)$, and (ii) exhibit \emph{low proximity} \citep{ustun2019actionable,guo2023counternet}, i.e., the change from $x$ to $\cf$ is small.

\textbf{Model Extraction Attack via CF Explanations.}
We consider the scenario where a malicious adversary  (i.e., attacker) aims to perform a model extraction attack on a proprietary machine learning model $F_W$. The attacker operates under the following assumptions:
(i) \emph{Black-box Access}: The attacker has black-box access to the proprietary model $F_W$ and the CF explanation method $\verb|CF|(\cdot;F_W)$, i.e., they can query the model with input $x$, and obtain the corresponding output probability $F_W(x)$, the corresponding CF explanations' input-output pairs $(\cf, F_W(\cf))$, where $\cf = \verb|CF|(x; F_W)$.
(ii) \emph{Attack Dataset Generation}: The attacker queries $F_W$ with $M$ distinct attack points, denoted as $\{(x_i)\}_{i=1}^M$, to construct a supervised attack dataset $D^x = {(x_i,F_W(x_i))}_{i=1}^M$.
(iii) \emph{Model Extraction}: The attacker leverages the attack dataset $D^x$ to train an extracted ML model, $f_{w}:\mathcal{X} \to [0,1]$, that closely approximates the behavior of the proprietary model, i.e., $\{f_{w}(x_i) \approx F_W(x_i) | \ \forall x_i \in D^x\}$.

Prior work has demonstrated that adversaries can exploit CF explanations to improve the query efficiency of model extraction attacks \cite{aivodji2020model, wang2022dualcf}. For example, \emph{MRCE} \citep{aivodji2020model} incorporates a set of corresponding CF explanations $D^\text{cf}=\{(\cf_i, 1 - F_W(x_i))\}^{M}$ into the attack dataset, and use both $D^x$ and $D^\text{cf}$ for training their extracted model. \emph{DualCF} \citep{wang2022dualcf} further improves the query efficiency of \emph{MRCE} (see Appendix~\ref{appendix:attacks} for details).
In this paper, we address this vulnerability by introducing a watermarking approach to protect against such attacks.

\section{CFMark: A Watermarking Framework for CF Explanations}
\label{sec:watermark}

We propose CFMark, a novel watermarking framework for CF explanations to protect against model extraction attacks that use CF explanations. At a high level, CFMark consists of two stages: (i) \emph{watermark embedding}: we design a watermarking function $G_\theta$ that can be used to embed watermarks into the CF explanations that are generated by the CF explanation method used by the defender. Specifically, $G_\theta$ inputs a CF explanation $\cf$ and outputs a watermarked CF explanation $G_\theta(\cf)$. These watermarked CF explanations can later be detected if malicious users use these watermarked CFs as training data for unauthorized model extraction attacks. Importantly, unlike prior methods that aggressively perturb the CF explanations, CFMark's perturbations have little impact on the utility of CF explanations.
(ii) \emph{watermark detection}: we employ a pairwise t-test to identify any third-party black-box ML model trained on our watermarked CF explanations.

\subsection{Stage 1: Watermark Embedding}
In this paper, we define a $\theta$-perturbation to the input CF explanation as our watermarking function $G_\theta(\cf) = \cf + \theta$, where the perturbation $\theta$ possesses the same dimension as $\cf$ (i.e., $\cf_i \in  \mathbb{R}^d$, and $\theta_i \in \mathbb{R}^d$). 
Crucially, the selection of $\theta$ involves optimizing two key objectives:
(i) \emph{detectability:} the perturbation $\theta$ should maximize the likelihood of successfully detecting an extracted ML model trained on watermarked CF explanations;
(ii) \emph{usability:} the perturbation $\theta$ should minimize the degradation of watermarked CF explanations' quality. This degradation is quantified by the change in validity induced by the $\theta$-perturbation watermark.
Note that proximity is implicitly preserved as the perturbation $\theta$ is constrained within an $l_p$-norm ball (as discussed later), which inherently limits the impact on the proximity of watermarked CF explanations.
% Note that we omit proximity as part of the usability objective because the perturbation $\theta$ is constrained within an $l_p$-norm ball (as discussed later), which inherently limits the impact on the proximity of watermarked CF explanations.

% should aim to achieve minimal degradation in the quality of watermarked counterfactual explanations. This degradation is measured through the change in validity resulting from the $\theta$-perturbation watermark, i.e., the validity of the watermarked CF explanations should remain as similar as possible to the validity of unwatermarked CF explanations.

% that our watermarking function $G_\theta(\cf) = \cf + \theta$ adds a $\theta$-perturbation to the input CF explanation $\cf$ ($\theta$ is a perturbation vector of the same size as $\cf$). This $\theta$-perturbation is chosen to simultaneously optimize two key objectives: (i) \emph{detectability} - we want to maximize our chances of accurately detecting an extracted ML model that has been trained on the watermarked CF explanations; and (ii) \emph{usability} - we want to minimize the loss in quality of CF explanations (as measured by change in validity) caused by the addition of this watermark (aka $\theta$-perturbation).

% We introduce a watermarking function $G_\theta(\cf) = \cf + \theta$ that operates by adding a perturbation vector $\theta$ to an input counterfactual explanation $\cf$. The perturbation $\theta$ possesses the same dimensionality as $\cf$. Crucially, the selection of $\theta$ is guided by the simultaneous optimization of two key objectives:

\textbf{Bi-Level Optimization for Watermarking.} 
Let $D^x = \{(x_i, F_W(x_i)) \}^M_i$ denote an initial attack set, queried by adversaries from the proprietary ML model $F_W$ using $M$ attack points.
% where adversaries query the proprietary ML model $F_W$ using $M$ attack points. 
In addition, the adversaries have access to the corresponding \emph{watermarked} CF explanations and their predictions $D^\text{cf}(\boldsymbol{\theta}) = \{(\cf_i + \theta_i, F_W(\cf_i + \theta_i)| \mbox{ }\forall\mbox{ } (x_i,y_i) \in D^x\}$.
We denote the \emph{unwatermarked} CF explanations and their predictions as $D^\text{cf} = \{(\cf_i, F_W(\cf_i)| \mbox{ }\forall\mbox{ } (x_i,y_i) \in D^x\}$.
As per our black-box access assumption, the adversary cannot access unwatermarked CF explanations.
Given $D^x$ and $D^\text{cf}(\boldsymbol{\theta})$, the adversaries can train an extracted model:% $f_{w^*}$:
% Let $f_{w^*}$ denote an extracted ML model that is trained by using an initial attack set $(x,y)\in D^x$ (that is used to query the defender's ML model $F_W$) and the corresponding watermarked CF explanations and their predictions that are generated by the defender $D^\text{cf} = \{(\cf + \theta, 1-y)| \mbox{ }\forall\mbox{ } (x,y) \in D_x\}$.  
\begin{equation}
    w^* = \arg\min_{w} \frac{1}{N} \sum_{(x_i, y_i) \in D^x \cup D^\text{cf}(\boldsymbol{\theta})} \mathcal{L}(f_w(x_i), y_i) 
\end{equation}
Crucially, this extracted ML model $f_{w^*}$ is trained on the watermarked CF explanations $D^\text{cf}(\boldsymbol{\theta})$, which makes it possible to detect watermarks at a later stage. Intuitively, if the extracted model is trained on the watermarked explanations, this model should exhibit greater confidence in classifying watermarked explanations than unwatermarked explanations. %, which are excluded from the training set.
Therefore, we formalize the \emph{detectability} objective as the \emph{maximization of the logarithmic difference} between model outputs from watermarked and unwatermarked CF explanations from the extracted ML models $f_{w^*}$ (i.e., $\log(f_{w^*}(\cf+\theta) - \log(f_{w^*}(\cf)$), referred as the \emph{poison loss}.
This proposed log difference loss leverages principles from information theory \citep{ash2012information}; maximizing this loss encourages higher certainty (or less \emph{surprises}) in classifying watermarked explanations while lowering the certainty in classifying unwatermarked counterparts. 
% The perturbation $\theta$ is chosen to maximize this poison loss.
% the probability output of $f_{w^*}$ on the watermarked CF explanation $f_{w^*}(\cf+\theta)$ should be higher than the probability output of $f_{w^*}$ on the un-watermarked CF explanation $f_{w^*}(\cf)$.
% we \emph{maximize the entropy differences of models' outputs of watermarked and unwatermarked CF explanations}, i.e., $\sum_{(x,y) \in D^\text{cf}} \log\left(\frac{f_{w^*}(x+\theta)}{f_{w^*}(x)}\right)$, as our \emph{detectability} objective (we refer to this as the \textit{poison loss}). 

% Then, we formalize our \emph{detectability} objective as follows: (i) the probability output of $f_{w^*}$ on the watermarked CF explanation $f_{w^*}(\cf+\theta)$ should be higher than the probability output of $f_{w^*}$ on the un-watermarked CF explanation $f_{w^*}(\cf)$. (ii) to maximize \emph{detectability}, we should be choosing $\theta$ that maximizes the difference between the predictions returned (by the extracted ML model) on the watermarked ($f_{w^*}(\cf+\theta)$) vs un-watermarked CF explanations ($f_{w^*}(\cf)$). To satisfy these considerations, we \emph{maximize the entropy differences of models trained on watermarked and unwatermarked CF explanations}, i.e., $\sum_{(x,y) \in D^\text{cf}} \log\left(\frac{f_{w^*}(x+\theta)}{f_{w^*}(x)}\right)$, as our \emph{detectability} objective (we refer to this as the \textit{poison loss}). 
Furthermore, to maintain the quality of CF explanations, we aim to minimize the deviation between watermarked and unwatermarked CF explanations on the original proprietary ML models $F_W$ (i.e., $F_W(\cf +\theta) \approx F_W(\cf)$.
Hence, we formalize the \emph{usability} objective by minimizing the Kullback-Leibler (KL) divergence between the probability outputs from the proprietary model $F_W$ (i.e., $KL \left( F_W(\cf+\theta) \parallel F_W(\cf) \right)$), referred as the \textit{validity loss}). This KL-divergence term ensures similar predictions between the watermarked and unwatermarked CF explanations, thereby ensuring the quality of watermarked CF explanations.
% unlike in the poison loss where divergent predictions between watermarked and unwatermarked counterfactual (CF) explanations are desirable, 
% minimizing the KL-divergence ensures that the quality of predictions achieved by watermarked CF explanations is similar to that achieved by unwatermarked CF explanations, thereby ensuring minimal changes to usability of CF explanations.

%the usability objective encourages minimizing the divergence between the probability distributions of $F_W(\cf+\theta)$ and $F_W(\cf)$. This objective motivates our use of the KL divergence as a measure of distribution similarity.

% Next, we formalize the \emph{usability} objective as follows: (i) to maximize \emph{usability}, we should choose $\theta$ such that if the un-watermarked CF explanation ($\cf$) was valid w.r.t. the proprietary ML model $F_W$, then the watermarked CF explanation ($\cf + \theta$) should also remain valid on $F_W$. (ii) We formalize this notion by \emph{minimizing} $\sum_{(x,y) \in D^\text{cf}} \text{KL} \left( F_W(x+\theta) \parallel F_W(x) \right)$, the KL-divergence between the probability outputs returned by $F_W$ on the watermarked vs un-watermarked CF explanations (we refer to this as the \textit{validity loss}). 
Thus, we formulate the watermarking embedding process as this bi-level optimization problem:
% \begin{equation}
%     \begin{aligned}
%         \max_{\boldsymbol{\theta}, \forall \theta_i\in \Delta} \frac{1}{N} \sum_{{(x_i. y_i) \in D^\text{cf}}} &\lambda_1\log\left(\frac{f_{w^*}(x_i+\theta_i)}{f_{w^*}(x_i)}\right)
%     - \lambda_2 {KL} ( F_W(x_i+\theta_i) \parallel F_W(x_i)), \nonumber\\
%     \end{aligned}
% \end{equation}
% \begin{equation}
%     \text{s.t.,} \  w^* = \arg\min_{w} \frac{1}{N} \!\!\!\! \sum_{(x_i, y_i) \in D^x \cup D^\text{cf}(\boldsymbol{\theta})} \!\!\!\!\mathcal{L}(f_w(x_i), y_i).\label{eq:watermark1}
% \end{equation}
\begin{equation}
    \begin{aligned}
        \max_{\theta} & \sum_{(x,y) \in D^\text{cf}} \underbrace{\log\left(\frac{f_{w^*}(x+\theta)}{f_{w^*}(x)}\right)}_\text{Poison Loss}
        - \underbrace{\text{KL} \biggl( F_W(x+\theta) \parallel F_W(x) \biggr)}_{\text{Validity Loss}}, \\ 
        \text{s.t.} \quad & w^*(\theta) = \arg\min_{w} \sum_{(x,y) \in D^x} \mathcal{L}(f_w(x), y) + \sum_{(x,y) \in D^\text{cf}} \mathcal{L}(f_w(x+\theta), y)
    \end{aligned}
    \label{eq:watermark1}
\end{equation}
where $\lambda_1$ and $\lambda_2$ are hyperparameters that balance the two loss functions, $\Delta$ denotes the $l_\infty$-norm ball $\Delta=\{\theta_i \in \mathbb{R}^d | \norm{\theta_i}_\infty \leq \delta\}$, and $\delta$ is a hyperparamter which denotes the maximum perturbation. The inner (min) problem solves the model extraction problem where an adversary uses attack data points $D^x$ and the watermarked CF explanations $D^\text{cf}(\boldsymbol{\theta})$ as training data $D^x \cup D^\text{cf}(\boldsymbol{\theta})$ to extract the proprietary ML model. The outer (max) problem jointly optimizes the \emph{detectability} and \emph{usability} objectives for generating watermarking.

\subsubsection{Improving Generalization and Mitigating Overfitting}
Unfortunately, directly optimizing Eq.~\ref{eq:watermark1} leads to suboptimal generalization and a tendency to overfit (shown in Section~\ref{sec:eval}). To address this problem, we propose to augment this formulation via two key techniques: (i) \emph{regularization}, and (ii) \emph{data augmentation}.

% 1. extracted models are overfitting
% 2. overfitting on part of the optimization

%Next, we describe a regularization technique to mitigate the watermarking techniques from overfitting the extracted models. 

\textbf{Regularization.}
Our preliminary experiments optimizing the bi-level formulation in Eq.~\ref{eq:watermark1} show a \emph{high false positive} rate during the watermark detection stage.
In particular, our detection system (discussed in Section~\ref{sec:detection}) falsely flags benign models that are not trained using watermarked CF explanations. This is an undesirable behavior because it leads to potentially false alarms and misidentification of legitimately trained models; even if a model is trained using normal data collection procedures, it might be falsely flagged as an unauthorized extraction model.

This high false positive rate happens when the selected perturbations $\theta$ lead to high poison loss w.r.t. benign models (i.e., models not trained on watermarked CF explanations $D^\text{cf}(\boldsymbol{\theta})$). 
Essentially, the optimization of $\theta$ overfits to achieve high poison loss on any model, irrespective of whether it is trained on $D^\text{cf}(\boldsymbol{\theta})$ or not (as shown in Section~\ref{sec:eval}).
To mitigate this overfitting problem, we incorporate a regularization term to discourage the optimization of $\theta$ in achieving high poison loss 
exclusively for models trained on $D^\text{cf}(\boldsymbol{\theta})$, thereby reducing the false positive rate.
Specifically, this regularization term minimizes the log difference between model outputs from watermarked and unwatermarked CF explanations from the \emph{benign} ML models.

\textbf{Data Augmentation to Improve Generalizability.}
As an additional measure to prevent overfitting, we enrich the training data of the extracted ML model $f_{w}$ by including sampled data points from the defender's proprietary training dataset $D^t$ (which is feasible since the watermarking generation problem will be solved by the defender who has access to their training dataset).

\subsubsection{CFMark: Our Watermarking Algorithm}
Finally, we derive our new watermarking embedding:
% \begin{equation}
%     % \begin{aligned}
%     \max_{\boldsymbol{\theta}, \forall \theta_i\in \Delta} \frac{1}{N} \sum_{(x_i,y_i) \in D^\text{cf}} 
%     \lambda_1 \underbrace{\log\left(\frac{f_{w^*_1}(x_i+\theta_i)}{f_{w^*_1}(x_i)}\right)}_{\text{Poison Loss}} - \lambda_2 \underbrace{{KL} \biggl( F_W(x_i+\theta_i) \parallel F_W(x_i) \biggr)}_{\text{Validity Loss}} 
%     {\color{blue}- \lambda_3 \underbrace{\log\left(\frac{f_{w^*_2}(x_i + \theta_i))}{f_{w^*_2}(x_i)}\right)}_{\text{Regularization}}}, \label{eq:outer} 
%     % \end{aligned}
% \end{equation}
\begin{align}
    &\max_{\boldsymbol{\theta}, \forall \theta_i\in \Delta} \frac{1}{N} \sum_{(x_i,y_i) \in D^\text{cf}} 
    \lambda_1 \underbrace{\log\left(\frac{f_{w^*_1}(x_i+\theta_i)}{f_{w^*_1}(x_i)}\right)}_{\text{Poison Loss}} - \lambda_2 \underbrace{{KL} \biggl( F_W(x_i+\theta_i) \parallel F_W(x_i) \biggr)}_{\text{Validity Loss}} 
    {\color{blue}- \lambda_3 \underbrace{\log\left(\frac{f_{w^*_2}(x_i + \theta_i))}{f_{w^*_2}(x_i)}\right)}_{\text{Regularization}}}, \label{eq:outer} \\
    \text{s.t.} \quad & w_1^* = \arg\min_{w_1}  \!\!\!\! \!\!\!\!  \sum_{(x_i,y_i) \in D^x {\color{purple}\cup D^t} \cup D^\text{cf}(\boldsymbol{\theta})}   \!\!\!\! \!\!\!\! \mathcal{L}(f_{w_1}(x_i), y_i), \label{eq:w1}  \\
    \quad & {\color{blue}w^*_2 = \arg\min_{w_2} \sum_{(x_i,y_i) \in D^x \cup D^t} \mathcal{L}(f_{w_2}(x_i), y_i)}. \label{eq:w2}
\end{align}
where $\lambda_3$ are hyperparameters that balance the regularization strength. Compared to Eq.~\ref{eq:watermark1}, this new formulation highlights the two proposed techniques for mitigating overfitting.
First, we incorporate the \textcolor{blue}{regularization term} to mitigate the overfitting to poison loss on benign models $f_{w_2}$, which do not use $D^\text{cf}(\boldsymbol{\theta})$ for training. Furthermore, we \textcolor{purple}{augment the data} using sampled training data $D^t$ for training $f_{w_1}$ and $f_{w_2}$ to improve the generalizability of watermarks.

% where the inner (min) problem solves the model extraction problem faced by an adversary that tries to use attack data points along with the watermarked CF explanations as training data to extract the enterprise's ML model. On the other hand, the outer (max) problem simulates the two objectives required for good watermarking: (i) the first part of the objective maximizes the difference between the predictions returned (by the extracted ML model) on the watermarked vs un-watermarked CF explanations; and (ii) the second part of the objective is a KL divergence based objective that minimizes the difference in validity between the watermarked vs un-watermarked CF explanations (on the enterprise's ML model). 

The bi-level optimization problem in Eq.~\ref{eq:outer}-\ref{eq:w2} is generally intractable due to its nested structure.
Fortunately, we can efficiently approximate this bi-level formulation by alternating the optimization of the inner- and outer- problems using unrolling pipelines ~\citep{shaban2019truncated, gu2022min}, which has been applied to many ML problems with a bi-level formulation, e.g., meta-learning \citep{finn2017model}, poisoning attacks~\citep{huang2020metapoison}, and distributionally robust optimization \citep{guo2023rocoursenet}.
Algorithm~\ref{alg:watermark} details the optimization procedure for the watermarking embedding $\theta$ defined in Eq.~\ref{eq:outer}-\ref{eq:w2}.
We iteratively solve this bi-level optimization problem via $T$ outer steps. 
Each step begins by updating the weights of the extracted models $w_1$ and benign models $w_2$. This update is performed using $K$ unrolled gradient descent steps (Line~\ref{alg:line:unrolling_start}-\ref{alg:line:unrolling_end}).
Next, we maximize the outer objective function (Line~\ref{alg:line:outer}) and project $\theta$ into the feasible region $\Delta$ (Line~\ref{alg:line:proj}).
Importantly, when computing the gradient of the outer objective function w.r.t. the watermarking $\theta$ (Line~\ref{alg:line:outer}), we look ahead several forward steps in the inner problem (Line~\ref{alg:line:unrolling_start}-\ref{alg:line:unrolling_end}), then backpropagate the gradient to the initial unrolling step (Line~\ref{alg:line:outer}). 
This look-ahead mechanism traces the gradients back to the model unrolling stages, allowing for a more accurate approximation of $\theta$.
%This lookahead mechanism is employed because we use $K$ unrolling steps of gradient descent rather than executing gradient descent until convergence.

% each model using the watermarked data $\cf + \theta$ via $K$ unrolling steps of gradient descent.
% %Next, we adopt 
% Next, we maximize the adversarial loss and project $\theta$ into the feasible region $\theta$.
% Crucially, when calculating the gradient of adversarial loss (outer problem) with respect to data shift $\theta$, we look ahead in the inner problem for a few forward steps and then back-propagate to the initial unrolling step.
% We do this because we use $K$ unrolling steps of gradient descent, as opposed to full-blown gradient descent till convergence.

\begin{algorithm}[h]
\caption{Watermarking Algorithm}
\label{alg:watermark}
    \begin{algorithmic}[1]
    \State \textbf{Hyperparameters:} step size $\alpha$, \# of watermarking steps $T$, \# of unrolling steps $K$, maximum perturbation $\Delta$.
    \State \textbf{Input: } A batch of inputs $(x, y) \in {D}^x$ and their corresponding unwatermarked CF explanations $(x, y) \in {D}^\text{cf}$, and sampled training data $D^t$.
    \State \textbf{Initialize: } Init $\theta$ with zeros
    \For{$i = 1\rightarrow T$ steps}    
        \For{$k = 1 \rightarrow K$ unroll steps}\label{alg:line:unrolling_start}
            \State Update $w_1$ (Eq.~\ref{eq:w1}), $w_2$ (Eq.~\ref{eq:w2}) using Adam\label{alg:line:unrolling}
        \EndFor\label{alg:line:unrolling_end}
        \State $\boldsymbol{\theta} \leftarrow \boldsymbol{\theta} + \alpha \cdot sign\left(\nabla_{\theta} {L}\right)$ \Comment{$L := $ Eq.~\ref{eq:outer}} \label{alg:line:outer}
        \State Project $\boldsymbol{\theta}$ onto the $l_\infty$-norm ball. \label{alg:line:proj}
    \EndFor\\
    \Return $\boldsymbol{\theta}$
\end{algorithmic}
\end{algorithm}

\subsection{Stage 2: Ownership Verification}
\label{sec:detection}
This section details the verification process for determining if a suspicious model has been trained on watermarked CF explanations.
% We describe how to verify whether a suspicious model is trained on watermarked CF explanations. 
Assuming access to the predicted probability from the suspicious model,
we employ a null hypothesis significance testing (NHST) scheme (Proposition \ref{prop:ttest}) to identify unauthorized model extractions using watermarked CF explanations.
% Assuming query access to the suspicious model to obtain the predicted probability of any given input. We use the following null hypothesis significance testing (NHST) scheme (Proposition \ref{prop:ttest}) to identify unauthorized model extractions using watermarked CF explanations.

\begin{proposition}
\label{prop:ttest}
Suppose $p_x$ is the posterior probability of $x$ predicted by the suspicious model. 
Let $\cf$ and  $\hat{x}^{\text{cf}}$ each represent the unwatermarked and watermarked counterfactual explanations.
Let $\bar{p}_{x^{\text{cf}}}$ and  $\bar{p}_{\hat{x}^{\text{cf}}}$ each denote the mean of the posterior probabilities $p_{x^{\text{cf}}}$ and $p_{\hat{x}^{\text{cf}}}$ over $n$ observations. Given the null hypothesis $H_0: \bar{p}_{\hat{x}^{\text{cf}}} = \bar{p}_{x^{\text{cf}}} + \tau$ ($H_1: \bar{p}_{\hat{x}^{\text{cf}}}>\bar{p}_{x^{\text{cf}}} + \tau $), where $\tau$ is a hyper-parameter, we claim that the suspicious model is trained on counterfactual explanations (with $\tau$-certainty) if and only if $H_0$ is rejected.
\end{proposition}

In practice, we randomly sample $N$ data points from the test sets to conduct this pairwise t-test. We reject the null hypothesis $H_0$ if the resulting p-value is below a predetermined significance level $\alpha$ (we set $\alpha=0.05$).
Finally, we provide a theoretical analysis to establish the conditions for rejection of $H_0$ at significance level $\alpha$ in Theorem~\ref{theorem:hypo-test}. The proof is provided in Appendix~\ref{appendix:proof}.

\begin{theorem}
    \label{theorem:hypo-test}
    % Let $p_x$ denote the posterior probability of $x$ as predicted by the suspicious model. Let $\bar{p}_{x^{\text{cf}}}$ and  $\bar{p}_{\hat{x}^{\text{cf}}}$ each denote the mean of the posterior probabilities of $p_{x^{\text{cf}}}$ and $p_{\hat{x}^{\text{cf}}}$ over $n$ observations. 
    We define the following quantities: $\bar{d} = \mathbf{E}(p_{\hat{x}^\text{cf}} - p_{x^\text{cf}})$ and $\tilde{d} = \sum_{i=1}^{n}(p_{\hat{x_i}^\text{cf}} - p_{x_i^\text{cf}})^2$. We claim that defenders owners can reject the null hypothesis $H_0: \bar{p}_{\hat{x}^{\text{cf}}} = \bar{p}_{x^{\text{cf}}} + \tau$ (versus $H_1: \bar{p}_{\hat{x}^{\text{cf}}}>\bar{p}_{x^{\text{cf}}} + \tau $) at significance level $\alpha$, if $\bar{d}$ and $\tilde{d}$ satisfy that
    \begin{align}
    \sqrt{n^2-n}(\bar{d}-\tau) - t_{1-\alpha}\sqrt{\tilde{d}+n\bar{d}^2} >0
    \end{align}
    where $\tau$ is the level of certainty and  $t_{1-\alpha}$ is the (1-$\alpha$)-quantile of t-distribution with $n-1$ degrees of freedom and $n$ is the sample size of $x^{cf}$. %(See proof in Appendix~\ref{appendix:proof})
\end{theorem}

\section{Experimental Evaluation}
\label{sec:eval}

\textbf{Datasets.}
We evaluate the performance of our watermarking framework using four real-world datasets:
(i) \emph{Cancer} \citep{Dua:2019} contains 569 instances and uses cell nuclei characteristics to classify tumors as malignant (\verb|Y=1|) or benign (\verb|Y=0|). 
(ii) \emph{Credit} dataset \citep{yeh2009comparisons} contains 30,000 instances and predicts whether a borrower will default on their payments (\verb|Y=1|) or not (\verb|Y=0|) based on historical payment records.
(iii) \emph{HELOC} dataset \citep{heloc} collects anonymized Home Equity Line of Credit applications from real homeowners. This contains 10,459 instances, and the classifier predicts whether an applicant will repay their HELOC account within 2 years (\verb|Y=1|) or not (\verb|Y=0|) based on their application information.
(iv) \emph{Loan} dataset \citep{li2018should} contains $\sim$450k loan approval records across the U.S. from 1994 to 2009, and predicts whether a business defaulted on a loan (\verb|Y=1|) or not (\verb|Y=0|).

\textbf{Attacker Models \& CF Methods.}
We evaluate CFMark against three model extraction methods:
(i) \emph{Querying} attack \citep{tramer2016stealing} does not use CF explanations; instead it only uses the inputs and predictions pair $D^x$ for training the extracted ML model.
(ii) \emph{MRCE} \citep{aivodji2020model} adopts both inputs $D^x$ and the corresponding watermarked CF explanations for training the extracted ML model.
(iii) \emph{DualCF} \citep{wang2022dualcf} adopts both watermarked CF explanations $D^\text{cf}(\boldsymbol{\theta})$ and their dual CF explanations $D^\text{cf}(\boldsymbol{\theta})'$ for training the extracted ML model.
% Importantly, CFMark should consider extracted models trained using watermarked CF explanations (i.e., MRCE and DualCF) as positive instances, and those trained without as negative instances (i.e., Querying attacks).

% extracted models trained on watermarked CF explanations (e.g., MRCE, DualCF) are considered positive, while those trained without (e.g., Querying Attack) are negative.

% Importantly, both \emph{MRCE} and \emph{DualCF} can only access the watermarked CF explanations, as per our black-box access assumption, which restricts access to the unwatermarked CF explanations.%, the adversary cannot access unwatermarked CF explanations.

We use three widely used CF methods for benchmarking: 
(i) \emph{C-CHVAE} \citep{pawelczyk2020learning} is a parametric approach that generates CF explanations by perturbing the latent variables of a Variational Autoencoder (VAE) model.% until a valid CF example is found. 
(ii) \emph{DiCE} \citep{mothilal2020explaining} is a non-parametric method that generates diverse CF explanations. % by learning a distribution of potential CFs around the original instance. 
(iii) \emph{GrowingSphere} \citep{laugel2017inverse} is another non-parametric method that employs a random search algorithm to find valid recourses. %by sampling points around the input instance.

\begin{figure}
    \centering
    \includegraphics[width=0.8\columnwidth]{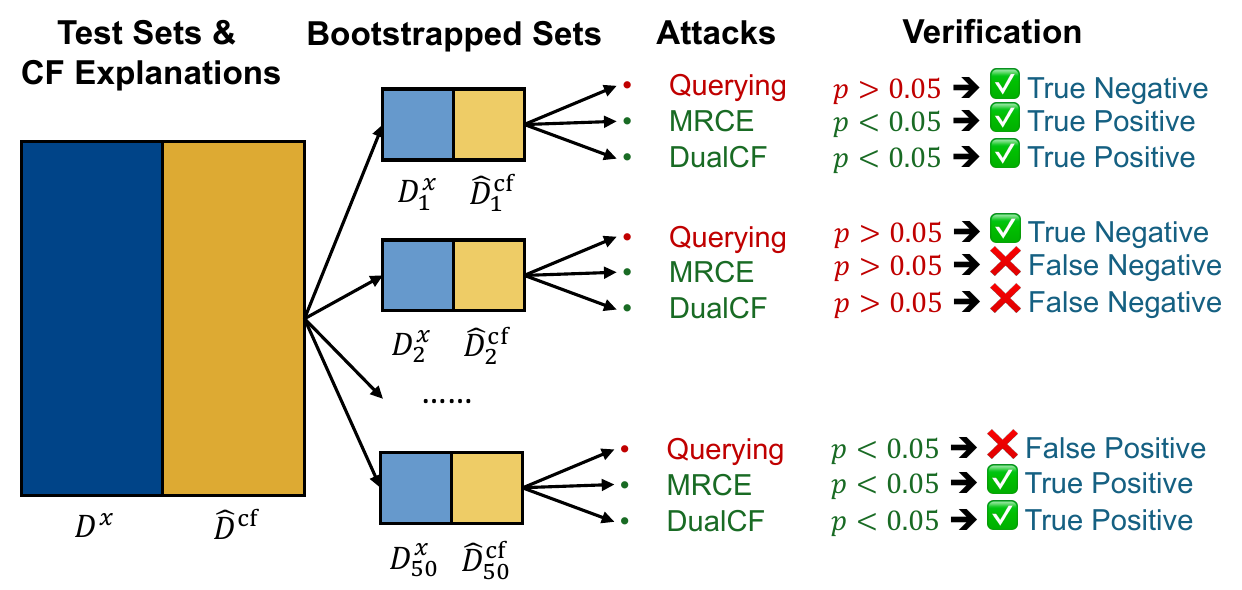}
    \caption{Illustration of the evaluation procedure for watermarking. \textcolor{teal}{Green} indicates positive cases, and \textcolor{red}{red} indicates negative cases.}
    \vspace{-5pt}
    \label{fig:eval-procedure}
\end{figure}

\textbf{Evaluation Procedure \& Metrics.}
Figure~\ref{fig:eval-procedure} illustrates the evaluation procedure of the experiment.
We use a bootstrapping approach to evaluate CFMark's ability to safeguard unauthorized model extraction attacks. 
Specifically, we first partition each dataset into train/test set splits. We train the proprietary model $F_W$ from the training set, and consider the test set as a potential attack set.
Next, we create 50 bootstrap subsets from the original test set. Each of these subsets was exposed to a model extraction attack. As a result, there are 50 extracted models for each attack method, CF explanation, and dataset combination. 
Extracted models trained on watermarked CF explanations (e.g., MRCE, DualCF) are considered positive, while those trained without (e.g., Querying Attack) are negative.
Next, for each extracted model, we perform the ownership verification procedure (in Proposition~\ref{prop:ttest}). Specifically, our framework outputs \emph{positive} (i.e., flags a model as being trained on watermarked CFs) if the $p$-value of the ownership verification is less than 0.05. Otherwise, we consider the output of our detection system as \textit{negative}.
We calculate true/false positives/negatives by comparing our ownership verification results with the ground-truth values of these attacked models. 
In total, we experiment with 1,800 extracted models ($50$ subsets $\times $ 3 attacks $\times$ 4 datasets $\times$ 3 CF methods) to rigorously quantify our watermarking framework's ability to identify unauthorized model extraction attacks.
% Extracted models trained on watermarked CF explanations (e.g., MRCE, DualCF) are considered positive, while those trained without (e.g., Querying Attack) are negative.

To evaluate the effectiveness of CFMark, we report the \emph{F-1} score on measuring how accurate CFMark is in identifying models trained on watermarked or unwatermarked CF explanations.
% Finally, we report two metrics:
% (i) \emph{true rate (TR)}, which calculates the percentage of correctly identifying positives (i.e., models that use MRCE and DualCF) and negatives (i.e., models that use Querying attack), and (ii) the \emph{F-1} score.
% In total, we experiment 1,800 bootstrap subsets ($50$ subsets $\times $ 3 attacks $\times$ 4 datasets $\times$ 3 CF methods) to rigorously quantify our watermarking framework's ability to identify unauthorized model extraction attacks.
% 1. D = D_train, D_test
% 2. For each D_test, bootstrap 50 folds
Furthermore, we use two widely used metrics to evaluate the quality of CF explanations \citep{wachter2017counterfactual,mothilal2020explaining, guo2023counternet}: 
(i) \emph{Validity}, which measures the fraction of valid CF explanations $\cf$ with respect to $F_W$;
(ii) \emph{Proximity}, which computes the $l_1$ distance between the input $x$ and its CF explanation $\cf$.

\begin{table}[t]
\centering
\small
\caption{\label{tab:result}F1 score of CFMark in identifying model extraction attacks using watermarked CF explanations. }
\setlength{\tabcolsep}{3.5pt}
\begin{tabular}{l|cccc}
\toprule
\textbf{CF Method} & \multicolumn{1}{l}{\textbf{Cancer}} & \multicolumn{1}{l}{\textbf{Credit}} & \multicolumn{1}{l}{\textbf{HELOC}} & \multicolumn{1}{l}{\textbf{Loan}} \\ \midrule\midrule
C-CHVAE & 0.90 & 0.52 & 0.87 & 0.90 \\
DiCE & 0.95 & 0.98 & 0.93 & 0.93 \\
Growing Sphere & 1.00 & 0.92 & 0.89 & 0.83 \\ \bottomrule
\end{tabular}
% \vspace{-10pt}
\end{table}

% 1. bi-level watermark vs random watermark
% 2. watermark on various CF methods and datasets against various attacks

\subsection{Evaluation Results}
\label{sec:eval-results}
\textbf{Watermarking Performance.}
Table~\ref{tab:result} shows the effectiveness of CFMark in identifying model extraction attacks using CF explanations. 
In particular, across different datasets and CF methods, CFMark can accurately identify true positives and true negatives, achieving an average F1-score of 0.95, 0.91, and 0.80, across DiCE, Growing Sphere, and C-CHVAE, respectively.
This result underscores CFMark's performance in identifying unauthorized usage of CF explanations in model extraction attacks, while not misjudging models that do not use CF explanations for training models.

% Table~\ref{tab:result} shows that our model ownership verification is highly effective across all datasets, CF methods, and attack methods. Specifically, our method can accurately identify unauthorized usage of CF explanations in model extraction attacks with high confidence (i.e., p-value$\ll 0.01$) across both MRCE and DualCF attacks. Additionally, our method does not misjudge models that do not use CF explanations for training models, as evident in the query attack, where the p-value is close to 1.
% On the other hand, random perturbations to CF explanations are clearly ineffective as they achieve a high p-value in all cases.

% \begin{wraptable}{r}{0.5\textwidth}
\begin{table}[t]
% \small
% \vspace{-5pt}
\centering
\scriptsize
\caption{\label{tab:cf}Evaluation of the CF Explanations. Watermarked CF explanations (i.e., WM.) achieve comparative validity (i.e., Val.) and proximity (i.e., Prox.) as their unwatermarked counterparts (i.e., Original).}
\setlength{\tabcolsep}{1.5pt}
\begin{tabular}{ll|ll|ll|ll|ll}
\toprule
\multicolumn{2}{l|}{} & \multicolumn{2}{c|}{\textbf{Cancer}} & \multicolumn{2}{c|}{\textbf{Credit}} & \multicolumn{2}{c|}{\textbf{HELOC}} & \multicolumn{2}{c}{\textbf{Loan}} \\
\multicolumn{2}{l|}{\multirow{-2}{*}{\textbf{CF Method}}} & \cellcolor[HTML]{DAE8F5}Val. & \cellcolor[HTML]{D8F0D3}Prox. & \cellcolor[HTML]{DAE8F5}Val. & \cellcolor[HTML]{D8F0D3}Prox. & \cellcolor[HTML]{DAE8F5}Val. & \cellcolor[HTML]{D8F0D3}Prox. & \cellcolor[HTML]{DAE8F5}Val. & \cellcolor[HTML]{D8F0D3}Prox. \\ \midrule\midrule
 & \cellcolor[HTML]{FFFFFF}Original & \cellcolor[HTML]{FFFFFF}1.0 & \cellcolor[HTML]{FFFFFF}3.23 & \cellcolor[HTML]{FFFFFF}1.0 & \cellcolor[HTML]{FFFFFF}4.40 & \cellcolor[HTML]{FFFFFF}1.0 & \cellcolor[HTML]{FFFFFF}3.73 & \cellcolor[HTML]{FFFFFF}1.0 & \cellcolor[HTML]{FFFFFF}6.90 \\
 & \cellcolor[HTML]{FFFFFF}WM. & \cellcolor[HTML]{FFFFFF}0.99 & \cellcolor[HTML]{FFFFFF}3.27 & \cellcolor[HTML]{FFFFFF}0.98 & \cellcolor[HTML]{FFFFFF}4.42 & \cellcolor[HTML]{FFFFFF}0.94 & \cellcolor[HTML]{FFFFFF}3.43 & \cellcolor[HTML]{FFFFFF}1.0 & \cellcolor[HTML]{FFFFFF}6.92 \\
\multirow{-3}{*}{CCHVAE} & \cellcolor[HTML]{EFEFEF}Change (\%) & \cellcolor[HTML]{EFEFEF}1.01 & \cellcolor[HTML]{EFEFEF}1.23 & \cellcolor[HTML]{EFEFEF}2.02 & \cellcolor[HTML]{EFEFEF}0.45 & \cellcolor[HTML]{EFEFEF}6.19 & \cellcolor[HTML]{EFEFEF}8.38 & \cellcolor[HTML]{EFEFEF}0 & \cellcolor[HTML]{EFEFEF}0.29 \\ \midrule
 & \cellcolor[HTML]{FFFFFF}Original & \cellcolor[HTML]{FFFFFF}0.99 & \cellcolor[HTML]{FFFFFF}4.21 & \cellcolor[HTML]{FFFFFF}.90 & \cellcolor[HTML]{FFFFFF}3.98 & \cellcolor[HTML]{FFFFFF}0.90 & \cellcolor[HTML]{FFFFFF}8.39 & \cellcolor[HTML]{FFFFFF}0.67 & \cellcolor[HTML]{FFFFFF}9.01 \\
 & \cellcolor[HTML]{FFFFFF}WM. & \cellcolor[HTML]{FFFFFF}0.99 & \cellcolor[HTML]{FFFFFF}4.34 & \cellcolor[HTML]{FFFFFF}.90 & \cellcolor[HTML]{FFFFFF}4.05 & \cellcolor[HTML]{FFFFFF}0.89 & \cellcolor[HTML]{FFFFFF}8.43 & \cellcolor[HTML]{FFFFFF}0.67 & \cellcolor[HTML]{FFFFFF}9.03 \\
\multirow{-3}{*}{DiCE} & \cellcolor[HTML]{EFEFEF}Change (\%) & \cellcolor[HTML]{EFEFEF}0 & \cellcolor[HTML]{EFEFEF}3.04 & \cellcolor[HTML]{EFEFEF}0 & \cellcolor[HTML]{EFEFEF}1.74 & \cellcolor[HTML]{EFEFEF}1.12 & \cellcolor[HTML]{EFEFEF}0.48 & \cellcolor[HTML]{EFEFEF}0 & \cellcolor[HTML]{EFEFEF}0.22 \\ \midrule
 & \cellcolor[HTML]{FFFFFF}Original & \cellcolor[HTML]{FFFFFF}1.0 & \cellcolor[HTML]{FFFFFF}3.50 & \cellcolor[HTML]{FFFFFF}0.99 & \cellcolor[HTML]{FFFFFF}5.18 & \cellcolor[HTML]{FFFFFF}1.0 & \cellcolor[HTML]{FFFFFF}3.67 & \cellcolor[HTML]{FFFFFF}1.0 & \cellcolor[HTML]{FFFFFF}7.57 \\
 & \cellcolor[HTML]{FFFFFF}WM. & \cellcolor[HTML]{FFFFFF}0.98 & \cellcolor[HTML]{FFFFFF}3.52 & \cellcolor[HTML]{FFFFFF}0.98 & \cellcolor[HTML]{FFFFFF}5.21 & \cellcolor[HTML]{FFFFFF}0.98 & \cellcolor[HTML]{FFFFFF}3.73 & \cellcolor[HTML]{FFFFFF}1.0 & \cellcolor[HTML]{FFFFFF}7.62 \\
\multirow{-3}{*}{\begin{tabular}[c]{@{}l@{}}Growing\\ Sphere\end{tabular}} & \cellcolor[HTML]{EFEFEF}Change (\%) & \cellcolor[HTML]{EFEFEF}2.02 & \cellcolor[HTML]{EFEFEF}0.57 & \cellcolor[HTML]{EFEFEF}1.02 & \cellcolor[HTML]{EFEFEF}0.58 & \cellcolor[HTML]{EFEFEF}2.02 & \cellcolor[HTML]{EFEFEF}1.62 & \cellcolor[HTML]{EFEFEF}0 & \cellcolor[HTML]{EFEFEF}0.66 \\ \bottomrule
\end{tabular}
\vspace{-5pt}
% \end{wraptable}
\end{table}

\textbf{Validity \& Proximity.}
Table~\ref{tab:cf} compares the validity and proximity of original (i.e., unwatermarked) $\cf$ and watermarked CF explanations $\hat{x}^\text{cf}$. It shows that watermarking CF explanations only leads to a minor degradation in quality. Specifically, compared with their unwatermarked counterparts, watermarked CF explanations exhibit an average 1.3\% decrease in validity and a 1.6\% increase in proximity across all datasets. 
This negligible degradation suggests that watermarking CF explanations provides a robust defense mechanism while preserving the utility of the explanations.

\begin{figure}[t!]
     \centering
     \begin{subfigure}[h]{0.49\columnwidth}
         \centering
         \includegraphics[width=0.98\textwidth]{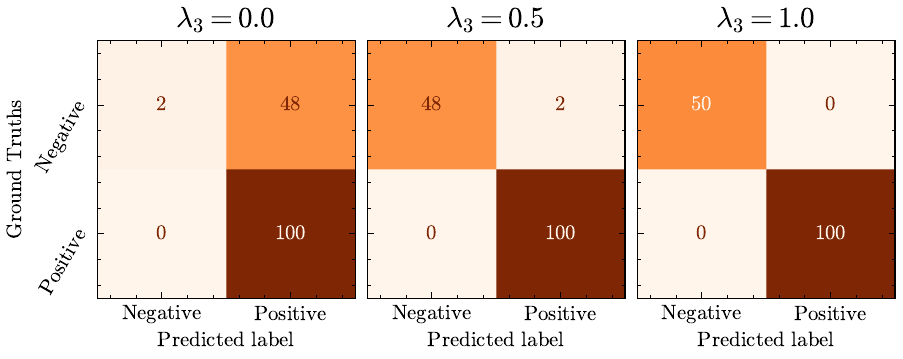}
         \caption{Confusion matrix when varying $\lambda_3$.}
         \label{fig:ablation-reg-cm}
     \end{subfigure}
     \hfill
     \begin{subfigure}[h]{0.49\columnwidth}
         \centering
         \includegraphics[width=0.98\textwidth]{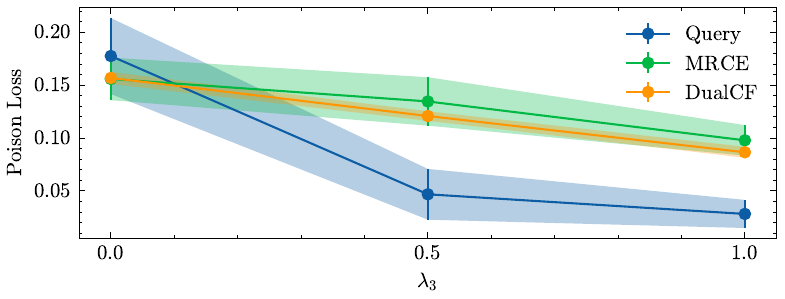}
         \vspace{-5pt}
         \caption{Poison loss when varying $\lambda_3$.}
         \label{fig:ablation-reg-plot}
     \end{subfigure}
     \caption{The impact of the \emph{regularization} term on the credit dataset when using DiCE.}
     \vspace{-5pt}
    \label{fig:ablation-reg}
\end{figure}

\begin{table*}[t]
\caption{Ablations of the loss functions of CFMark on \emph{credit} dataset when using DiCE. Val. (\%) and Prox. (\%) measure the validity decreases and proximity increases from unwatermarked to watermark CF explanations, respectively. High F1, low Val. (\%), and low Prox. (\%) are desirable.} \label{tab:ablation-loss}
\scriptsize
\centering
\begin{tabular}{@{}l|ccccccccc@{}}
\toprule
 & \multicolumn{9}{c}{\textbf{Validity Loss}} \\ \cmidrule(l){2-10} 
 & \multicolumn{3}{c|}{\textit{Log Diff}} & \multicolumn{3}{c|}{\textit{KL}} & \multicolumn{3}{c}{\textit{Residual}} \\
\multirow{-3}{*}{\textbf{Poison Loss}} & \cellcolor[HTML]{FEE9D4}F1 & \cellcolor[HTML]{DAE8F5}Val. (\%) & \multicolumn{1}{c|}{\cellcolor[HTML]{D8F0D3}Prox. (\%)} & \cellcolor[HTML]{FEE9D4}F1 & \cellcolor[HTML]{DAE8F5}Val. (\%) & \multicolumn{1}{c|}{\cellcolor[HTML]{D8F0D3}Prox. (\%)} & \cellcolor[HTML]{FEE9D4}F1 & \cellcolor[HTML]{DAE8F5}Val. (\%) & \cellcolor[HTML]{D8F0D3}Prox. (\%) \\ \midrule
\textit{Log Diff} & 0 & 0.02 & \multicolumn{1}{c|}{1.24} & \cellcolor[HTML]{EFEFEF}0.97 & \cellcolor[HTML]{EFEFEF}-0.30 & \multicolumn{1}{c|}{\cellcolor[HTML]{EFEFEF}1.76} & 0 & -0.23 & 1.42 \\
\textit{KL} & 0 & 4.49 & \multicolumn{1}{c|}{-0.20} & 0.00 & 1.44 & \multicolumn{1}{c|}{0.69} & 0 & 5.09 & -0.18 \\
\textit{Residual} & 0 & 0.32 & \multicolumn{1}{c|}{0.69} & 1.0 & -0.31 & \multicolumn{1}{c|}{1.90} & 0 & 0.07 & 1.32 \\ \bottomrule
% \vspace{-5pt}
\end{tabular}\end{table*}

\subsection{Further Analysis}
Due to space constraints, our ablation analysis focuses on evaluating the Credit dataset when using DiCE. Further ablation studies are included in the Appendix.

\textbf{The Impact of Regularization.}
Figure~\ref{fig:ablation-reg} illustrates the importance of the \emph{regularization} term in Eq.~\ref{eq:outer} by varying the trade-off parameter $\lambda_3$.
As shown in Figure~\ref{fig:ablation-reg-cm}, excluding the regularization term (i.e., $\lambda_3 = 0$) results in a high false positive rate of 96\% (48/50), which suggests that Query attacks are falsely identifying as positive cases.
This high false positive issue is due to overfitting of the poison loss when $\lambda_3=0$, where the poison loss for Query attacks is as high as that for MRCE and DualCF (Figure~\ref{fig:ablation-reg-plot}). This result highlights the overfitting challenges when optimizing the watermark $\theta$ without regularization.

On the other hand, increasing the value of $\lambda_3$ in Figure~\ref{fig:ablation-reg} leads to a significant decrease in the false positive rate and poison loss of Query attacks.  In particular, the false positive rate is drastically dropped to 4\% (2/50) when $\lambda_3 = 0.5$, and is further dropped to 0 when $\lambda_3 = 1.0$ (Figure~\ref{fig:ablation-reg-cm}).
This highlights the effectiveness of regularization in mitigating overfitting.
Notably, this regularization term has minimal impact on the detectability of watermarked CF explanations, as evidenced by the consistent perfect true positive rate even with increasing $\lambda_3$ values.

% On the other hand, when increasing the value of $\lambda_3$, Figure~\ref{fig:ablation-reg} shows a drastic decreased false positive rate and poison loss of Query attack. 
% In particular, the false positive rate is drastically decreased to 0.04 (2/50) when $\lambda_3 = 0.5$, and reduce to 0 when $\lambda_3 = 0.5$ (see Figure~\ref{fig:ablation-reg-cm}).
% This result demonstrates the effectiveness of introducing the regularization term in mitigating the overfitting issue.
% At the same time, the regularization term has little impact on the \emph{detectability} of the watermarked CF explanations. In particular, we observe a perfect true positive rate, even when increasing the regularization strength $\lambda_3$.

\textbf{The Impact of Data Augmentation.}
Figure~\ref{fig:ablation-dt} shows the impact of data augmentation (i.e., using sampled training data $D_t$ to train the extracted model) on poison loss against the MRCE attack. This figure shows that using data augmentation helps achieve significantly better performance against the C-CHVAE explanation method, while it does not make a significant difference against DiCE and Growing Sphere.

% \textbf{Understanding the Impact of Regularization \& Data Augmentation.}
% We demonstrate the importance of regularization trade-off parameter $\lambda$ in Figure \ref{fig:reg-credit-cchvae}. The X-axis shows varying $\lambda$ values, and Y-axis shows $P_{\Delta}$ achieved with DiCE explanation method. This figure shows that when $\lambda$ is set to 0 (i.e., no regularization), the Query attack also achieves a high $P_{\Delta}$ value (due to overfitting), due to which our watermarking detection scheme outputs a false positive. Increasing values of $\lambda$ decreases the $P_{\Delta}$ value for the Query attack faster as compared to MRCE and DualCF attacks, thereby illustrating that our regularization penalty helps alleviate the problem of false positives. Similarly, Figure \ref{fig:data-credit} shows the impact of data augmentation (i.e., using sampled training data $D_t$ to train extracted model) on $P_{\Delta}$ against the MRCE attack. This figure shows that using data augmentation really helps achieve significantly better performance against the C-CHVAE explanation method, while it does not make a significant difference against DiCE (DiverseCF) and Growing Sphere.

% \begin{figure}
%     \centering
%     \includegraphics[width=0.9\columnwidth]{figs/ablation/reg-dice-credit.pdf}
%     \caption{The impact of the \emph{regularization} on credit dataset when using DiCE.}
%     \label{fig:ablation-reg}
% \end{figure}

\begin{figure}[t!]
     \centering
     \includegraphics[width=0.8\columnwidth]{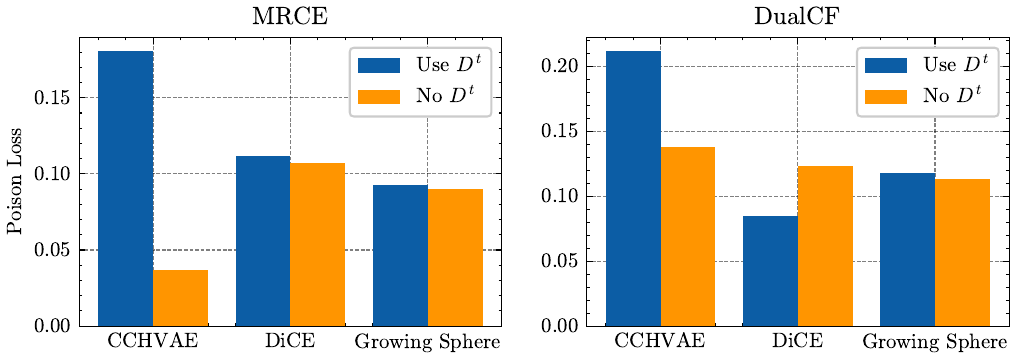}
     \caption{The impact of using training data $D^t$ in CFMark when evaluated on credit dataset.}
    \label{fig:ablation-dt}
    \vspace{-5pt}
\end{figure}

\textbf{Ablations of Loss Functions.}
We analyze three loss function ablations to highlight the design choice of loss functions in Eq.~\ref{eq:outer}.
In particular, both the poison and validity loss can adopt alternate functional forms. We experiment with three functional forms for each objective: (i) the logarithmic difference (i.e., $\log(f_{w^*}(x+\theta) - \log(f_{w^*}(x))$), used for poison loss and the regularization term; (ii) the KL Divergence (i.e., $\text{KL}(f_{w^*}(x+\theta) \parallel f_{w^*}(x)$), used for the validity term; and (iii) the residual difference (i.e., $f_{w^*}(x+\theta) - f_{w^*}(x)$), another reasonable loss function in our setting.
Table~\ref{tab:ablation-loss} reports the F1-score and the usability degradation (as measured by validity decreases (in \%), and proximity increases (in \%)) on different loss functions. The actual loss function combination used in the paper 
% (Log Diff for poison loss, KL for usability loss) 
is represented in Gray. 
Importantly, compared with other ablations, our choice of loss functions achieves a high F1 score (i.e., 0.97) while maintaining minimal degradation in validity and proximity (i.e., less than 2\%).
Furthermore, while using residual for poisoning and KL divergence for validity loss slightly improves F1 scores, it also leads to slightly increased proximity.

% We further demonstrate the choice of loss functions 
% using the logarithmic difference for poison loss and the regularization term, as compared to other common choices, such as KL Divergence (i.e., $\text{KL}(f_{w^*}(x+\theta) \parallel f_{w^*}(x)$) or the residual difference (i.e., $f_{w^*}(x+\theta) - f_{w^*}(x)$). 
% Figure~\ref{fig:ablation-loss} illustrates the impact of loss functions to $P_\Delta$. This figure highlights that using the entropy loss leads to low $P_\Delta$ for query extraction attacks, leading to low false negative detections. Furthermore, using the entropy loss leads to high $P_\Delta$ on both MRCE and DualCF attack, which indicates that this function leads to a high true positive case. 

\begin{figure}[t!]
     \centering
     \vspace{-5pt}
     \begin{subfigure}[h]{0.49\columnwidth}
         \centering
         \includegraphics[width=0.98\textwidth]{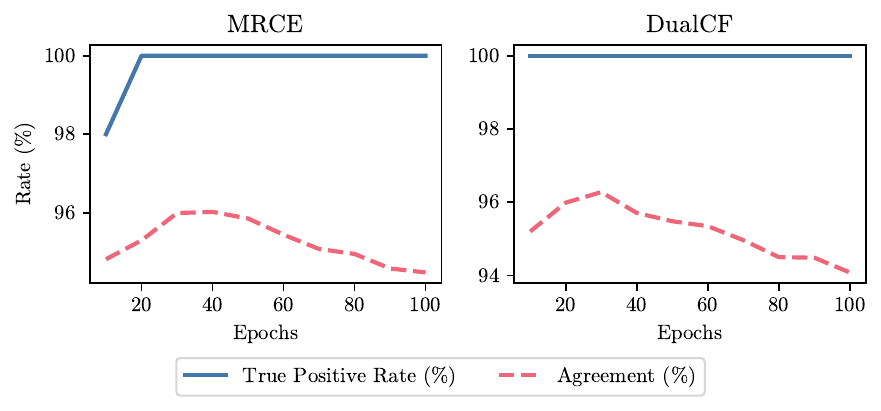}
         \caption{Robustness of CFMark to fine-tuning.}
         \label{fig:rob-ft}
     \end{subfigure}
     \hfill
     \begin{subfigure}[h]{0.49\columnwidth}
         \centering
         \includegraphics[width=0.98\textwidth]{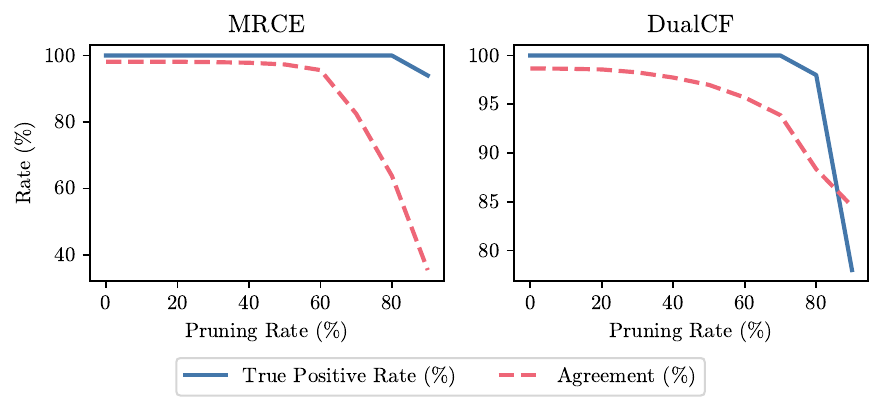}
         \caption{Robustness of CFMark to model pruning.}
         \label{fig:rob-prune}
     \end{subfigure}
     \caption{Robustness of CFMark to backdoor removal.}
     % \vspace{-5pt}
    \label{fig:rob}
\end{figure}

\textbf{Robustness to Backdoor Defense.}
We further study the robustness of CFMark against potential attacks by adversaries to remove these watermarks to avoid future detectability. Prior research has demonstrated that fine-tuning \citep{liu2017neural, liu2018fine} and model pruning \citep{wu2021adversarial, liu2018fine} are two common watermark removal techniques. Therefore, we evaluate the resilience of CFMark's watermarks against these two methods, as demonstrated in Figure~\ref{fig:rob}. We use two key metrics: \emph{true positive rate}, which measures the effectiveness of watermark detection after applying the watermark removal technique, and \emph{agreement}, which measures the percentage of times the extracted model $f_w$ and proprietary model $F_W$ are in agreement (i.e., they output the same labels). This \emph{agreement} metric reflects the effectiveness of the model extraction attack.

Figure~\ref{fig:rob-ft} shows that CFMark watermarks are robust to model fine-tuning. The near-perfect true positive rate for both MRCE and DualCF attacks indicates strong resistance to this removal technique, even with increasing epochs of fine-tuning. Similarly, Figure~\ref{fig:rob-prune} highlights CFMark's robustness to model pruning when increasing the pruning rate.
This figure shows that CFMark achieves a consistently high true positive rate (over 75\%) for MRCE and DualCF attack detection even under reasonably high pruning rates ($\sim$60\%). Further, we observe that increasing the pruning rate to 60\% deteriorates the agreement metric; however, we argue that attackers would typically avoid high pruning rates due to the resulting utility loss in the extracted model. This result further underscores the robustness of CFMark to backdoor removal techniques.

% \begin{figure}[t]
%     \centering
%     \includegraphics[width=0.98\columnwidth]{figs/robust/ft-credit.pdf}
%     \caption{Robustness of CFMark to fine-tuning. }
%     \label{fig:rob-ft}
% \end{figure}

% \begin{figure}[t]
%     \centering
%     \includegraphics[width=0.98\columnwidth]{figs/robust/pruning-credit.pdf}
%     \caption{Robustness of CFMark to model pruning. }
%     \label{fig:rob-prune}
% \end{figure}

% \begin{figure}[t]
%     \centering
%     \includegraphics[width=0.98\textwidth]{figs/ablation/reg/credit.pdf}
%     \caption{The influence of the regularization term $\lambda$ on the \emph{credit} dataset.}
%     \label{fig:reg-credit}
% \end{figure}

% \begin{figure}[t]
%     \centering
%     \includegraphics[width=0.98\textwidth]{figs/ablation/data/credit.pdf}
%     \caption{The impact of data augmentation.}
%     \label{fig:reg-credit}
% \end{figure}

% \begin{figure*}[ht!]
%      \centering
%      \begin{subfigure}[h]{0.98\textwidth}
%          \centering
%          \includegraphics[width=0.98\textwidth]{figs/ablation/reg/dummy.pdf}
%          \caption{}
%          \label{fig:reg-dummy}
%      \end{subfigure}
%      \hfill
%      \begin{subfigure}[h]{0.98\textwidth}
%          \centering
%          \includegraphics[width=0.98\textwidth]{figs/ablation/reg/credit.pdf}
%          \caption{}
%          \label{fig:reg-credit}
%      \end{subfigure}
%      \hfill
%      \caption{   }
%     \label{fig:ab-reg}
% \end{figure*}

% \section{Discussion \& Conclusion}
\section{Discussion \& Conclusion}
% limitations and social impact (needed because of the paper checklist)
\label{sec:discussion}

In this paper, we propose CFMark, the \emph{first} watermarking framework for counterfactual explanations to identify unauthorized model extraction attacks. We formulate this watermarking framework as a bi-level optimization problem, which embeds an indistinguishable watermark into the CF explanations. 
These watermarks can be subsequently detected using a pairwise t-test to identify unauthorized usage of CF explanations in extracting proprietary models.
Furthermore, we establish a theoretical foundation for the effectiveness of our verification procedure.
% We propose to use a pairwise t-test to verify the model ownership.
Empirical results show that CFMark can achieve robust identifiability, without compromising the quality of CF explanations.

\bibliography{ref}
\bibliographystyle{apalike}

%%%%%%%%%%%%%%%%%%%%%%%%%%%%%%%%%%%%%%%%%%%%%%%%%%%%%%%%%%%%

\appendix

\section{Proof}
\label{appendix:proof}

\begin{theorem}
    Let $p_x$ denote the posterior probability of $x$ as predicted by the suspicious model. Let $\cf$ and  $\hat{x}^{\text{cf}}$ each represent the unwatermarked and watermarked counterfactual explanations. Let $
    \bar{p}_{x^{\text{cf}}}$ and  $\bar{p}_{\hat{x}^{\text{cf}}}$ each denote the mean of the posterior probabilities of $p_{x^{\text{cf}}}$ and $p_{\hat{x}^{\text{cf}}}$ over $n$ observations. Define the following quantities: $\bar{d} = \mathbf{E}(p_{\hat{x}^\text{cf}} - p_{x^\text{cf}})$ and $\tilde{d} = \sum_{i=1}^{n}(p_{\hat{x_i}^\text{cf}} - p_{x_i^\text{cf}})^2$. We claim that dataset owners can reject the null hypothesis $H_0: \bar{p}_{\hat{x}^{\text{cf}}} = \bar{p}_{x^{\text{cf}}} + \tau$ (versus $H_1: \bar{p}_{\hat{x}^{\text{cf}}}>\bar{p}_{x^{\text{cf}}} + \tau $) at significance level $\alpha$, if $\bar{d}$ and $\tilde{d}$ satisfy that
    \begin{align}
    \sqrt{n^2-n}(\bar{d}-\tau) - t_{1-\alpha}\sqrt{\tilde{d}+n\bar{d}^2} >0
    \end{align}
    where $\tau$ is the level of certainty and  $t_{1-\alpha}$ is the (1-$\alpha$)-quantile of t-distribution with $n-1$ degrees of freedom and $n$ is the sample size of $x^{cf}$.
\end{theorem}

\begin{proof}
    Since both $p_{\hat{x}^{\text{cf}}} $ and $p_{x^{\text{cf}}} $ have finite means and variances. Define the distance between them as $d=p_{\hat{x}^{\text{cf}}} -p_{x^{\text{cf}}}$, which also has a finite mean and variance. Given that the data points are independent, the Central Limit Theorem (CLT) ensures the sample mean of distance, $\bar{d}$, converges to a Gaussian distribution as the sample size $n $ becomes sufficiently large. Thus, we can restate the original hypothesis as:
    
     $$H_0: \bar{d}-\tau = 0$$
     $$H_1: \bar{d}-\tau > 0$$
    
    Let $\tilde{d} = \sum_{i=1}^nd_i^2$, we can construct the t-statistic as follows:
    \begin{align}
        T:= \frac{\sqrt{n}(\bar{d}-\tau)}{\sigma_d} \sim t(n-1)
    \end{align}
    where $\sigma_d$ is the standard deviation of $\bar{d}$ and $\bar{d}-\tau$, i,e.,
    \begin{align}
        \label{eq:var_d_bar}
        \sigma_d^2 &= \frac{1}{n-1}{\sum_{i=1}^{n}(d_i-\bar{d})^2} = \frac{1}{n-1}(\tilde{d}+n\bar{d}^2)
    \end{align}
    To reject the hypothesis $H_0$ at the significance level $\alpha$, we need to ensure that
    \begin{align}
        \label{eq:rej_region}
        \frac{\sqrt{n}(\bar{d}-\tau)}{\sigma_d} > t_{1-\alpha}
    \end{align}
    where $t_{1-\alpha}$ is the $(1-\alpha)$-quantile of t-distribution with $(n-1)$ degree of freedom. \\
    According to equation \ref{eq:var_d_bar} and \ref{eq:rej_region}, we have:
    \begin{align}
        \sqrt{n^2-n}(\bar{d}-\tau) - t_{1-\alpha}\sqrt{\tilde{d}+n\bar{d}^2} >0
    \end{align}
\end{proof}

\section{Details about Model Extraction Methods via CF explanations}
\label{appendix:attacks}
In this section, we describe \emph{MRCE} \citep{aivodji2020model} and \emph{DualCF} \citep{wang2022dualcf}, two approaches that use CF explanations to perform model extraction attacks.

% \textbf{\emph{MRCE.}} 
\textbf{MRCE.}
\citet{aivodji2020model} showed that adversaries can perform a query-efficient model extraction attack by making half as many queries to $F_W$. Instead of using an attack dataset $D_x$ of size $M$, the MRCE attack proceeds by using an $M/2$ sized attack dataset $D^{MRCE}_x=\{(x_i,F_W(x_i)\}^{M/2}_i$ and uses the corresponding CF explanations and their predictions $D^\text{cf}=\{(\cf_i, 1 - F_W(x_i))\}^{M/2}_i$ as additional training data which does not require querying the ML model (since the labels on CF explanation points $\cf$ are assumed to be opposite to that of the original points $x$). Thus, leveraging CF explanations and their labels allows the attacker to get an $M$-sized training dataset by making only $M/2$ queries to $F_W$. Finally, the attacker use both $D^{MRCE}_x$ and $D^\text{cf}$ for training their extracted ML model $f_{w'}:\mathcal{X} \to [0,1]$.

% \textbf{\emph{DualCF.}} 
\textbf{DualCF.}
Alternatively, \citet{wang2022dualcf} improves the quality of the training dataset used by the attacker to train the extracted ML model $f_{w'}$. The DualCF attack proceeds as follows: (i) the attacker queries $F_W$ with an initial set of attack points to get $D^{DCF}_x=\{(x_i,F_W(x_i)\}^{M/2}_i$ and corresponding CF explanations $D^\text{cf}=\{(\cf_i, 1 - F_W(x_i))\}^{M/2}_i$. (ii) The set $D^\text{cf}$ is also used to query $F_W$ to get a dual CF dataset $D^\text{cf'}=\{(x^{cf'}_i, 1-(1-F_W(x_i)\}^{M/2}_i$ (i.e., the set of CF explanations $x^{cf'}$ for the original CF explanations $x^{cf}$ generated by the defender's CF module). Intuitively, if the defender uses a high-quality CF explanation module, $D^\text{cf}$ and $D^\text{cf'}$ would represent a dataset with smaller margins, and hence enable more accurate recovery (or extraction) of underlying decision boundary of $F_W$. Finally, the attacker use both $D^\text{cf}$ and $D^\text{cf'}$ for training their extracted ML model $f_{w'}:\mathcal{X} \to [0,1]$.

\textbf{Protecting against Model Extraction Attacks.} 
As described in Section \ref{sec:related}, a common approach for mitigating model extraction attacks (in general) is to degrade the quality of the proprietary ML model $F_W$ to disincentivize the attacker from conducting such an attack (thereby protecting the model) \citep{tang2024modelguard}. In the context of MRCE \cite{aivodji2020model} and DualCF \cite{wang2022dualcf}, this approach would entail degrading the quality (as measured by widely used metrics such as validity and proximity \cite{verma2020counterfactual}) of the generated CF explanations such that it becomes unattractive for an attacker to use CF explanations as part of the training dataset used for model extraction. However, we argue that such an approach is unsatisfactory to use in real-world applications since lowering the quality of CF explanations negatively affects our ability to provide meaningful and actionable recourse to negatively affected end-users.

Therefore, in this paper, we adopt a digital watermarking approach \cite{li2022untargeted,song2017machine} to protect against these model extraction attacks that use CF explanations. Our watermarking approach can ensure that unauthorized model extraction attacks that use watermarked CF explanations can be easily identified by the defender, while maintaining the quality of the generated CF explanations.

\section{Implementation Details}
\label{app:implementation}

Here we provide implementation details of our proposed framework on three datasets listed in Section~\ref{sec:eval}. We provide the code, dataset, and experiment logs in the supplemental material, or be accessed in this repository: \codeurl.
% \url{https://www.dropbox.com/scl/fi/tmxpmzzsycwnthske4tbv/watermarking-cf.zip?rlkey=zs4uw93whaibtj8qcp7xe5b6b&st=dxcos06f&dl=0}.

\textbf{Software and Hardware Specifications}. 
All experiments are run using Python (v3.10.10) with jax (v0.4.20) \citep{jax2018github}, scikit-learn (v1.2.2) \citep{pedregosa2011scikit}, and jax-relax (v0.2.7) \citep{guo2023relax} for the implementations. All our experiments were run on an Ubuntu 22.04.4 LTS virtual machine on the Google Cloud Platform with an Nvidia V100 GPU.

\textbf{Feature Engineering.} 
We use the default feature engineering pipeline provided in jax-relax \citep{guo2023relax}. 
Specifically, for continuous features, we scale all feature values into the [0, 1] range. To handle the categorical features, we transform the categorical features into numerical representations via one-hot encoding.
Note that during the watermarking procedure, we treat the categorical features as immutable features, i.e., we do not add perturbations to the categorical features.

\textbf{Hyperparamters.}
For all three datasets and CF methods, we run $T=50$ steps for watermarking CF explanations, and set $E=0.05$ as the maximum perturbation. The step size $\alpha = 2.5 \times \delta / T$ (based on \cite{madry2018towards}) for solving the bi-level problem in Equation~\ref{eq:outer}-\ref{eq:w2}. On the attack side, the model extractors have a maximum of 128 queries for extracting models reported in Table~\ref{tab:result}.
In addition, Table~\ref{tab:params} provides a detailed overview of the hyperparameters used for each dataset and CF method.

\begin{table*}[h]
\centering
\caption{\label{tab:params}Hyperparameters for each dataset.}
% \begin{tabular}{@{}ll|l|c|c|c|l@{}}
% \toprule
% CF Method & Dataset & Batch Size & $k$ & learning rate & $\tau$ & Ensembels \\ \midrule
% \multirow{3}{*}{C-CHVAE} & Dummy & 64 & 10 & 0.01 & 0.1 & 8 \\  
%  & Credit &  16 & 10 & 0.01 & 0.1 & 8 \\ 
%  & HELOC &  128 & 5 & 0.1 & 0.05 & 32
%  \\ \midrule
% \multirow{3}{*}{DiCE} & Dummy &  128 & 10 & 0.03 & 0.05 & 8  \\ 
%  & Credit & 64 & 10 & 0.01 & 0.05 & 8  \\  
%  & HELOC & 64 & 5 & 0.1 & 0.1 & 8  \\ \midrule
% \multirow{3}{*}{\begin{tabular}[c]{@{}l@{}}Growing\\ Sphere\end{tabular}} & Dummy & 128 & 5 & 0.05 & 0.05 & 16 \\  
%  & Credit &  64 & 10 & 0.01 & 0.05 & 8 \\ 
%  & HELOC &  128 & 5 & 0.1 & 0.05 & 16 \\ \bottomrule
% \end{tabular}
\begin{tabular}{@{}ll|c|c|c|c|c@{}}
\toprule
\textbf{CF Method} & \textbf{Dataset} & \multicolumn{1}{l|}{\textbf{Batch Size}} & \multicolumn{1}{l|}{\textbf{k}} & \multicolumn{1}{l|}{\textbf{learning rate}} & \multicolumn{1}{l|}{\textbf{$\tau$}} & \multicolumn{1}{l}{\textbf{Ensembels}} \\ \midrule
\multirow{4}{*}{C-CHVAE} & Cancer & 64 & 5 & 0.03 & 0.05 & 32 \\
 & Credit & 16 & 10 & 0.01 & 0.05 & 8 \\
 & HELOC & 128 & 5 & 0.1 & 0.05 & 32 \\
 & Loan & 64 & 10 & 0.01 & 0.05 & 8 \\ \midrule
\multirow{4}{*}{DiCE} & Cancer & 64 & 10 & 0.005 & 0.1 & 16 \\
 & Credit & 64 & 10 & 0.01 & 0.05 & 8 \\
 & HELOC & 64 & 5 & 0.1 & 0.05 & 8 \\
 & Loan & 128 & 10 & 0.03 & 0.05 & 8 \\ \midrule
\multirow{4}{*}{Growing Sphere} & Cancer & 128 & 10 & 0.02 & 0.05 & 32 \\
 & Credit & 64 & 10 & 0.01 & 0.05 & 8 \\
 & HELOC & 128 & 5 & 0.1 & 0.05 & 16 \\
 & Loan & 128 & 5 & 0.05 & 0.05 & 16 \\ \bottomrule
\end{tabular}
\end{table*}

\begin{figure}
    \centering
    \includegraphics[width=0.98\columnwidth]{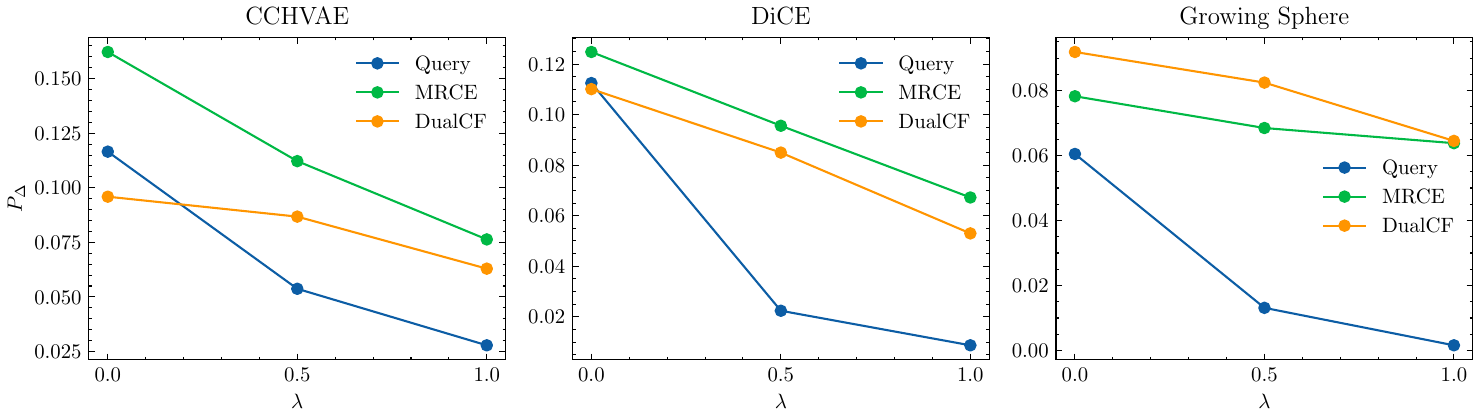}
    \caption{The influence of the regularization term $\lambda$ on the \emph{credit} dataset.}
    \label{fig:reg-credit}
\end{figure}

\begin{figure}[t]
    \centering
    \includegraphics[width=0.98\columnwidth]{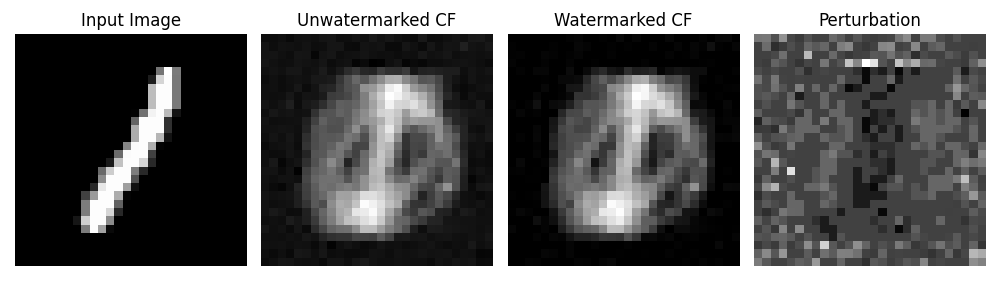}
    \caption{An example of watermarked CF explanation on MNIST.}
    \label{fig:mnist}
\end{figure}

\section{Additional Results}

\textbf{Additional Ablations on Loss Functions.}
We provide additional ablations on loss functions of Eq.~\ref{eq:outer}. Table~\ref{tab:ablation-loss-cchvae} reports the F1-score and the usability degradation (as measured by validity decreases (in \%), and proximity increases (in \%)) on different loss functions, when evaluating the Cancer dataset using C-CHVAE. We observe similar trend as we observed in the main paper.

\begin{table*}[t]
\caption{Ablations of the loss functions of CFMark on \emph{cancer} dataset when using C-CHVAE. } \label{tab:ablation-loss-cchvae}
\centering\small
\begin{tabular}{@{}lccccccccc@{}}
\toprule
 & \multicolumn{9}{c}{\textbf{Validity Loss}} \\ \cmidrule(l){2-10} 
 & \multicolumn{3}{c|}{\textit{Log Diff}} & \multicolumn{3}{c|}{\textit{KL}} & \multicolumn{3}{c}{\textit{Residual}} \\
\multirow{-3}{*}{\textbf{Poison Loss}} & \cellcolor[HTML]{FEE9D4}F1 & \cellcolor[HTML]{DAE8F5}Val. (\%) & \multicolumn{1}{c|}{\cellcolor[HTML]{D8F0D3}Prox. (\%)} & \cellcolor[HTML]{FEE9D4}F1 & \cellcolor[HTML]{DAE8F5}Val. (\%) & \multicolumn{1}{c|}{\cellcolor[HTML]{D8F0D3}Prox. (\%)} & \cellcolor[HTML]{FEE9D4}F1 & \cellcolor[HTML]{DAE8F5}Val. (\%) & \cellcolor[HTML]{D8F0D3}Prox. (\%) \\ \midrule
\multicolumn{1}{l|}{\textit{Log Diff}} & 0.13 & 9.73 & \multicolumn{1}{c|}{0.39} & \cellcolor[HTML]{EFEFEF}0.92 & \cellcolor[HTML]{EFEFEF}2.60 & \multicolumn{1}{c|}{\cellcolor[HTML]{EFEFEF}0.75} & 0.72 & 8.18 & 0.49 \\
\multicolumn{1}{l|}{\textit{KL}} & 0.04 & 12.07 & \multicolumn{1}{c|}{0.44} & 0.74 & 4.11 & \multicolumn{1}{c|}{0.82} & 0.21 & 12.15 & 0.43 \\
\textit{Residual} & 0.00 & 11.81 & 0.38 & 0.93 & 2.21 & 0.83 & 0.09 & 10.90 & 0.42 \\ \bottomrule
\end{tabular}
\end{table*}

\textbf{Confusion Matrix.}
Figure~\ref{fig:cm} highlights the detectability achieved by CFMark. Notably, CFMark achieves high True Positive and True Negative across all three CF methods and datasets.
This result highlights the effectiveness of CFMark in watermarking CF explanations.

\begin{figure*}[ht!]
     \centering
     \begin{subfigure}[h]{0.98\textwidth}
         \centering
         \includegraphics[width=0.98\textwidth]{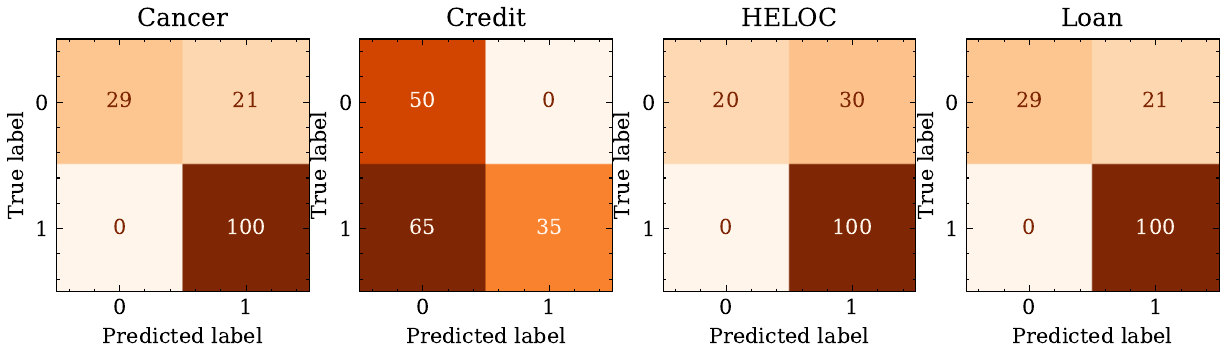}
         \caption{Confusion matrix of three datasets on CCHVAE.}
         \label{fig:cm-cchvae}
     \end{subfigure}
     \hfill
     \begin{subfigure}[h]{0.98\textwidth}
         \centering
         \includegraphics[width=0.98\textwidth]{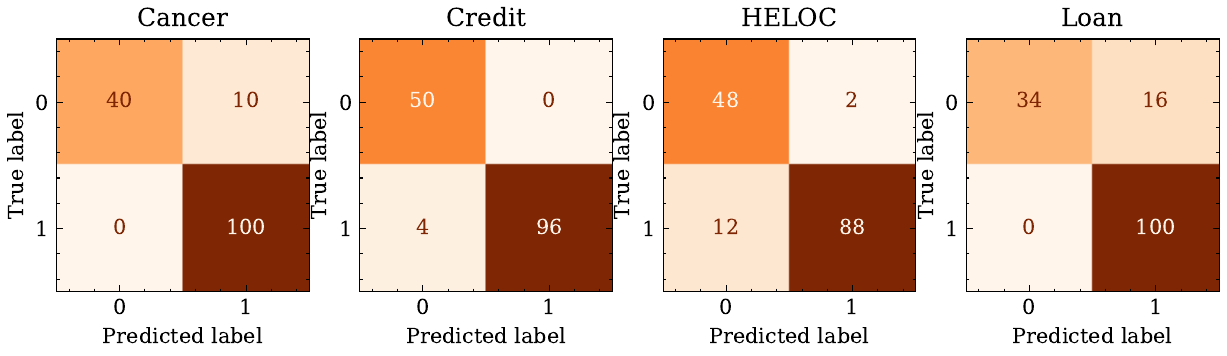}
         \caption{Confusion matrix of three datasets on DiCE.}
         \label{fig:cm-dice}
     \end{subfigure}
     \hfill
     \begin{subfigure}[h]{0.98\textwidth}
         \centering
         \includegraphics[width=0.98\textwidth]{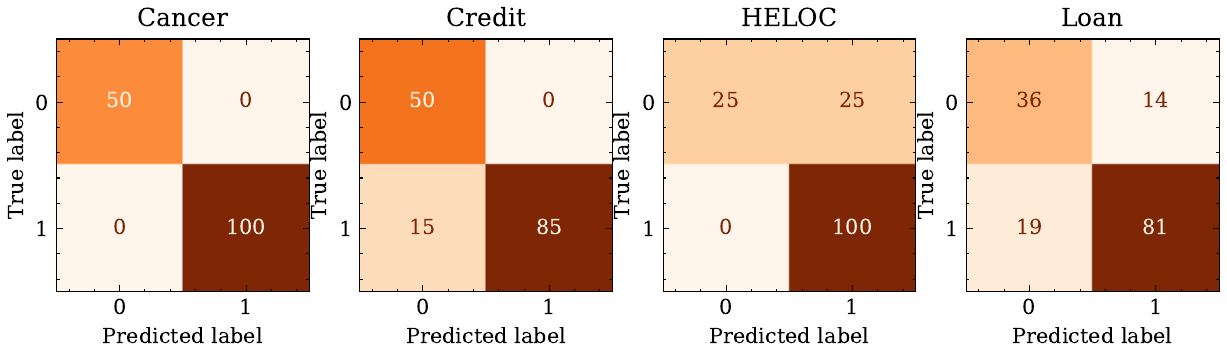}
         \caption{Confusion matrix of three datasets on Growing Sphere.}
         \label{fig:cm-growing}
     \end{subfigure}
     \caption{Confusion matrix of identifying unauthorized model extraction attacks through CF explanations from CCHVAE, DiCE, and Growing Sphere across three datasets.}
    \label{fig:cm}
\end{figure*}

\textbf{Loss curve.}
Figure~\ref{fig:loss} shows the loss curve of crafting watermarks on CF explanations generated from CCHVAE, DiCE, and Growing Sphere across three datasets. The watermarking procedure is stable.

\begin{figure*}[ht!]
     \centering
     \begin{subfigure}[h]{0.98\textwidth}
         \centering
         \includegraphics[width=0.98\textwidth]{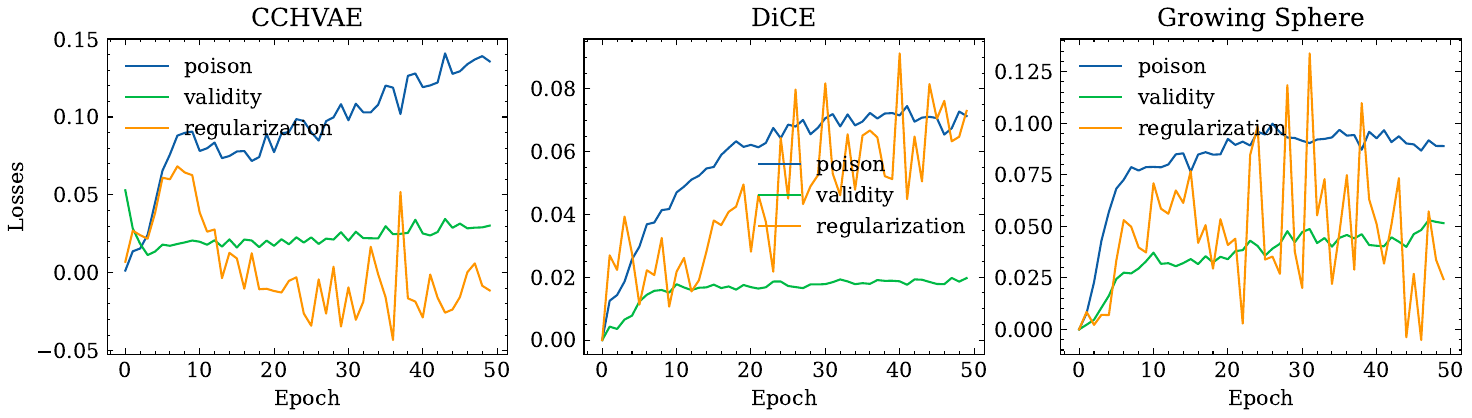}
         \caption{Loss curves on the \emph{Cancer} dataset.}
         \label{fig:loss-cancer}
     \end{subfigure}
     \hfill
     \begin{subfigure}[h]{0.98\textwidth}
         \centering
         \includegraphics[width=0.98\textwidth]{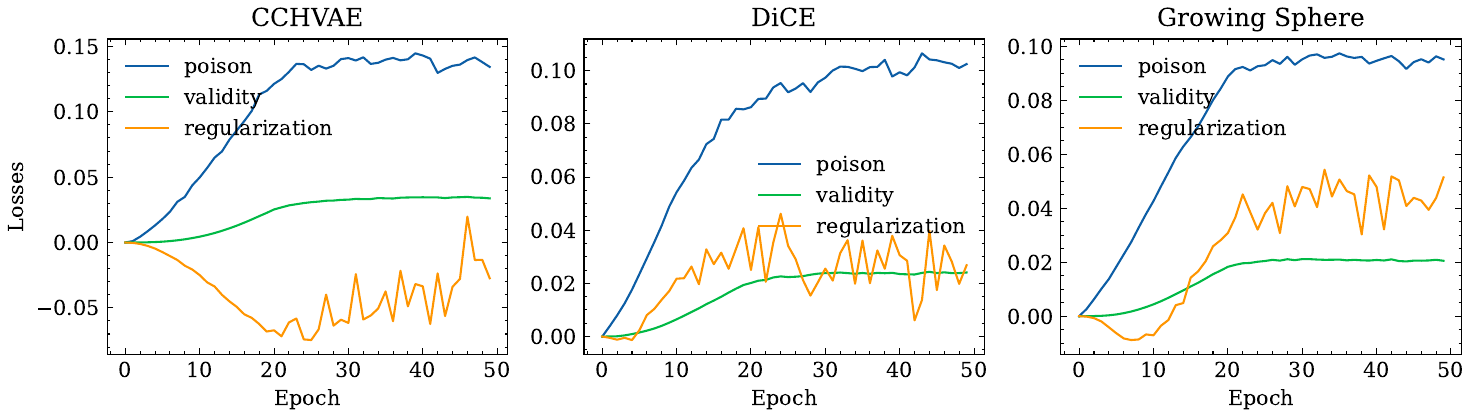}
         \caption{Loss curves on the \emph{Credit} dataset.}
         \label{fig:loss-credit}
     \end{subfigure}
     \hfill
     \begin{subfigure}[h]{0.98\textwidth}
         \centering
         \includegraphics[width=0.98\textwidth]{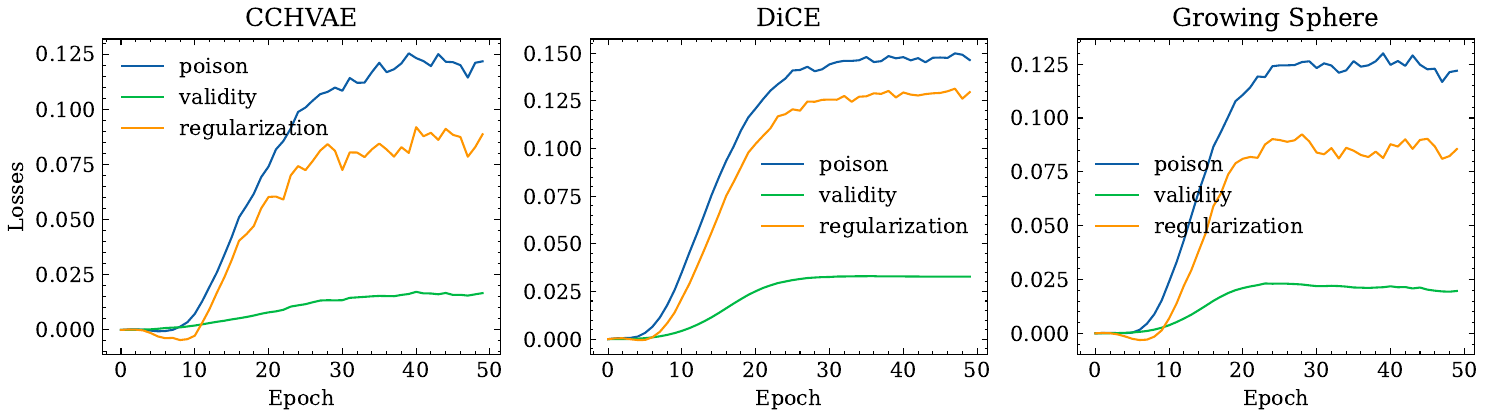}
         \caption{Loss curves on the \emph{HELOC} dataset.}
         \label{fig:loss-heloc}
     \end{subfigure}
     \hfill
      \begin{subfigure}[h]{0.98\textwidth}
         \centering
         \includegraphics[width=0.98\textwidth]{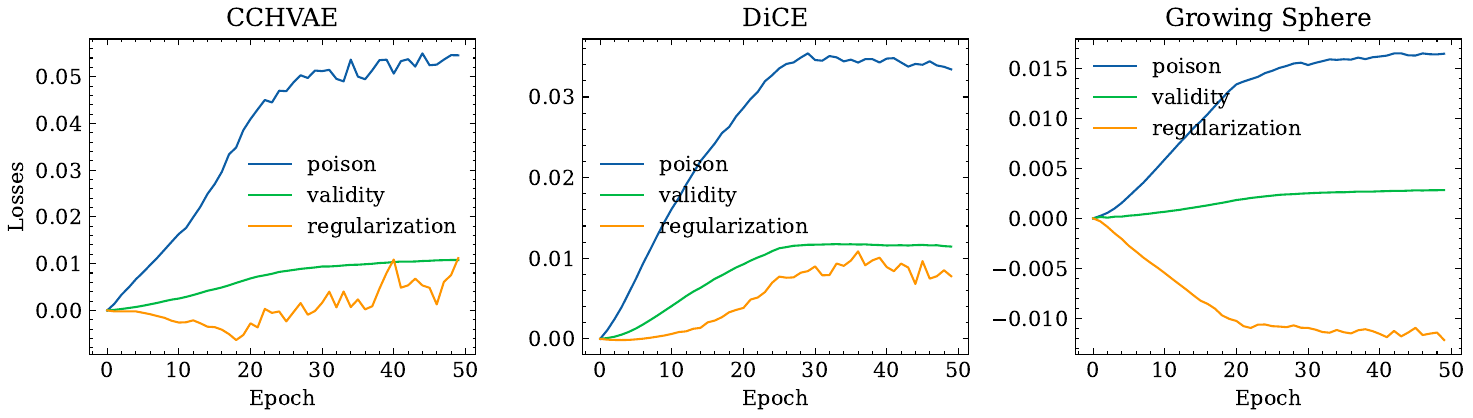}
         \caption{Loss curves on the \emph{Loan} dataset.}
         \label{fig:loss-loan}
     \end{subfigure}
     \caption{Loss curves of crafting watermarks on CF explanations generated from CCHVAE, DiCE, and Growing Sphere across three datasets.}
    \label{fig:loss}
\end{figure*}

\section{Case Study for Watermarking CF Explanations generated from C-CHVAE on the MNIST dataset.}

This section presents a case study on watermarking counterfactual (CF) explanations for image datasets. CF explanation techniques are not common choices for explaining image ML models. 
However, image data is valuable for understanding the impact of our watermarking technique because it allows for visual inspection and analysis of each watermarked CF explanation. This approach provides insights into how our watermarking technique influences the quality of CF explanations.

We experiment with the MNIST dataset \citep{deng2012mnist} for evaluating CF explanations. We focus on image digits of either 0 or 1. We use C-CHVAE to generate CF explanations.
Figure~\ref{fig:mnist} showcases an example of watermarked CF explanations produced by C-CHVAE. Notably, the watermarked CF explanations maintain the visual structure observed in their unwatermarked counterparts. This observation further demonstrates that the watermarking process has little impact on the quality of CF explanations.

\end{document}